%% file: main.tex
\documentclass{defaltmlr}
\input{math_commands.tex}

\usepackage{url}
\usepackage[table]{xcolor}
\usepackage[dvipsnames]{xcolor}
\usepackage{graphicx}
\usepackage{geometry}
\usepackage{booktabs}
\usepackage[T1]{fontenc}
\usepackage{fontawesome}
\usepackage{fancyhdr}
\usepackage{xspace}
\usepackage{enumitem}
\usepackage{etoc}
\usepackage{titletoc}
\usepackage{tcolorbox}
\usepackage{listings}

\usepackage{algorithm}
\usepackage{algpseudocode}

\usepackage{color}
\usepackage[dvipsnames]{xcolor}
\definecolor{JalapenoRed}{RGB}{183,21,64}
\definecolor{Belize}{RGB}{41,128,185}
\definecolor{Amour}{RGB}{238,82,83}
\usepackage{colortbl}

\usepackage{hyperref}

\usepackage{amsmath}
\usepackage{amssymb}
\usepackage{amsthm}
\usepackage{tocloft}

\usepackage[most]{tcolorbox}
\usepackage{enumitem}

\tcbset{
  promptbox/.style={
    title={#1},
    width=\columnwidth,
    breakable,
    enhanced,
    colback=gray!3,
    colframe=black!30,
    coltitle=black,
    fonttitle=\bfseries\footnotesize,
    boxrule=0.8pt,
    arc=1mm,
    boxsep=2pt,
    left=4pt,
    right=4pt,
    top=4pt,
    bottom=4pt,
    before skip=4pt,
    after skip=6pt
  }
}

\newcolumntype{L}[1]{>{\raggedright\arraybackslash}p{#1}}

\usepackage{titlesec}
\usepackage{titletoc}
\usepackage[most]{tcolorbox}
\usepackage{enumitem}

\usepackage[table]{xcolor}

\definecolor{GroundBG}{HTML}{6e8eb7}
\definecolor{ProfileBG}{HTML}{d79a77}
\definecolor{InteractBG}{HTML}{d77434}
\definecolor{OursBG}{HTML}{3e9b4e}

\usepackage[T1]{fontenc}
\usepackage[utf8]{inputenc}
\usepackage{microtype}
\usepackage{inconsolata}
\usepackage{graphicx}

\newcommand{\ie}{\textit{i}.\textit{e}.\xspace}

\usepackage[misc]{ifsym}
\usepackage{wrapfig}

\hypersetup{linkcolor=cyan}
\usepackage{mathtools}
\usepackage{ulem}
\usepackage{pifont}
\usepackage{bbding}

\usepackage{tcolorbox}
\usepackage{makecell}

\makeatletter
\let\orig@fnsymbol\@fnsymbol
\def\@fnsymbol#1{\ifcase#1\or\relax\else\orig@fnsymbol{#1}\fi}
\makeatother

\def\method{\texttt{Flash-OPD}}

\newtheorem{theorem}{Theorem}
\newtheorem{lemma}[theorem]{Lemma}
\newtheorem{proposition}[theorem]{Proposition}
\newtheorem{corollary}[theorem]{Corollary}

\newtheorem{remark}{Remark}

\usepackage{dsfont}

\usepackage{wrapfig}

\hypersetup{linkcolor=cyan}
\usepackage{mathtools}
\usepackage{ulem}
\usepackage{pifont}
\usepackage{bbding}

\usepackage{tcolorbox}
\usepackage{makecell}

\usepackage{pifont}
\newcommand{\cmark}{\ding{51}}
\newcommand{\xmark}{\ding{55}}

\title{Flash-OPD: Fast On-Policy Distillation}

\author{
\parbox{\textwidth}{
Wei Chen\textsuperscript{1,2,3}, \hspace{0.5mm}
Junle Chen\textsuperscript{3}, \hspace{0.5mm}
\textbf{Yitong Yang}\textsuperscript{1}, \hspace{0.5mm}
\textbf{Zhaoyang Xu}\textsuperscript{1}, \hspace{0.5mm}
\textbf{Jiaxin Lin}\textsuperscript{1}, \hspace{0.5mm}\\
\textbf{Yuxuan Liang}\textsuperscript{2}, \hspace{0.5mm}
\textbf{Xiaofang Zhou}\textsuperscript{3}, \hspace{0.5mm}
\textbf{Kai Wang}\textsuperscript{1}, \hspace{0.5mm}
\textbf{Rui Chen}\textsuperscript{1}
}
}

\affiliation{\textsuperscript{1}Tencent Hy, \hspace{0.5mm} \textsuperscript{2}HKUST(GZ), \hspace{0.5mm} \textsuperscript{3}HKUST
}

\abstract{
On-policy distillation (OPD) provides dense teacher supervision on student-generated trajectories, but generating and evaluating long rollouts incurs substantial training cost. Existing acceleration methods reduce this cost through open-loop rollout schedules or closed-loop horizon adaptation. However, supervision compatibility can vary substantially across trajectories, making a single rollout horizon difficult to match their heterogeneous reliable lengths: an overly short horizon may truncate useful supervision, while an overly long one wastes computation beyond reliable regions. Our key insight is that the trajectory-specific reliability boundary need not be predicted before generation. By viewing reliability as the first-passage of accumulated low teacher--student compatibility events, the boundary is inherently unknown before sampling, yet whether it has been reached can be determined exactly from the observed prefix. Building on this insight, we propose \method{}, which shifts from rollout-horizon control to adaptive trajectory-level boundary verification. \method{} interleaves cached generation with teacher verification and independently stops each trajectory according to its observed compatibility events. To reduce verification overhead, the recent event rate is used only to schedule the next verification point, while the actual stopping decision always relies on the exact cumulative count. This separation prevents estimation errors from causing premature termination while enabling efficient verification during generation. Extensive experiments across diverse datasets and teacher--student settings show that \method{} achieves $2.2\times$--$7.5\times$ speedups over standard OPD while maintaining or improving accuracy.
}

\correspondence{\textcolor{myblue}{onedeanxxx@gmail.com}}
\website{\url{https://onedean.github.io/Flash-OPD-Site}}
\repository{\url{https://github.com/Onedean/Flash-OPD}}
\date{September 24, 2026}

\begin{document}

\maketitle

\makeatletter
\let\@fnsymbol\orig@fnsymbol
\makeatother

\input{contents/1_introduction}

\input{contents/2_RelatedWork}

\input{contents/3_Preliminaries}

\input{contents/4_Methodology}

\input{contents/5_Experiments}

\input{contents/6_Conclusion}


\definecolor{textgray}{HTML}{6E6E73}
\makeatletter
\newcommand\applefootnote[1]{%
  \begingroup
  \renewcommand\thefootnote{}%
  \renewcommand\@makefntext[1]{\noindent##1}%
  \footnote{#1}%
  \addtocounter{footnote}{-1}%
  \endgroup
}
\makeatother

\bibliography{ref}
\bibliographystyle{plainnat}
\newpage
\appendix

\input{appendix/0_appendix_list}

\end{document}

%% file: math_commands.tex
\usepackage{amsmath,amsfonts,bm}

\def\eqref#1{equation~\ref{#1}}

\def\1{\bm{1}}

\DeclareMathAlphabet{\mathsfit}{\encodingdefault}{\sfdefault}{m}{sl}
\SetMathAlphabet{\mathsfit}{bold}{\encodingdefault}{\sfdefault}{bx}{n}



%% file: contents/1_introduction.tex
\section{Introduction}
\label{sec:intro}

On-policy distillation (OPD) has become an important component of frontier language model post-training, enabling both capability transfer and the consolidation of specialized teachers~\citep{yang2025qwen3,xu2026deepseek,zeng2026glm5}. By providing dense teacher feedback on student-generated prefixes, OPD mitigates the training--inference mismatch of learning from fixed teacher-generated responses~\citep{ross2011reduction,kim2016sequence,gu2024minillm}. Practical token-level formulations make this feedback locally available without accumulating future distillation signals into each update~\citep{agarwal2024policy,lu2025onpolicy}. However, local supervision does not automatically yield efficient long-horizon training: generating and evaluating long trajectories remains costly, while teacher guidance may become less useful as student prefixes drift into regions of low teacher--student compatibility~\citep{li2026rethinking,wang2026demystifying}.

Recent work addresses this bottleneck by controlling rollout length through predefined schedules~\citep{zhang2026full,zhang2026fast} or feedback-driven adaptation ~\citep{yang2026prune,liang2026adwin}. Yet these methods select a shared horizon for each batch, despite substantial variation in supervision compatibility across trajectories. Figure~\ref{fig.intro} illustrates this granularity mismatch: a shared horizon can extend well beyond the compatibility-based boundaries of some samples, even as it adapts over training. Shortening that horizon uniformly may instead curtail useful supervision for other samples. Sample-level prefix filtering offers finer selection, but produces a binary keep-or-drop decision rather than adapting the continuation length of each retained trajectory~\citep{zhao2026prefix}. These limitations motivate allocating rollout computation at the trajectory level.

Our key insight is that a trajectory's stopping boundary cannot be known exactly before sampling, nor does it need to be: it can instead be self-verified during generation. We use teacher--student top-$k$ overlap as a proxy for local compatibility and define the boundary as the first position where the cumulative number of low-compatibility events reaches a prescribed budget. This accumulation tolerates isolated disagreements while adapting the boundary to each realized trajectory. Although its location is unknown in advance, whether the boundary has been crossed can be determined exactly from the observed prefix. Moreover, with model parameters fixed during rollout collection, extending a prefix leaves its existing token-level distillation signals unchanged. This enables interleaved generation and teacher evaluation, turning shared horizon into per-trajectory boundary verification.

Building on this insight, we propose \method{}, which adaptively schedules verification and independently stops each trajectory within a reliable rollout horizon. At each verification point, \method{} evaluates newly generated positions and checks the cumulative event count. If the budget remains unexhausted, it uses the recent event rate to schedule the next check; otherwise, it stops. Crucially, the estimated event rate determines \textit{when to verify}, whereas the observed count determines \textit{whether to stop}. For boundaries reached before EOS or the length cap, adaptive verification cannot
trigger termination before the defined boundary. Resuming from cached prefixes ensures that each token is decoded only once, avoiding redundant generation. Through verified stopping at the trajectory level, \method{} reduces rollout computation and substantially improves training efficiency over vanilla OPD while maintaining strong performance. Our contributions are threefold:

\begin{figure*}[t]
    \centering
    \includegraphics[width=1.0\textwidth]{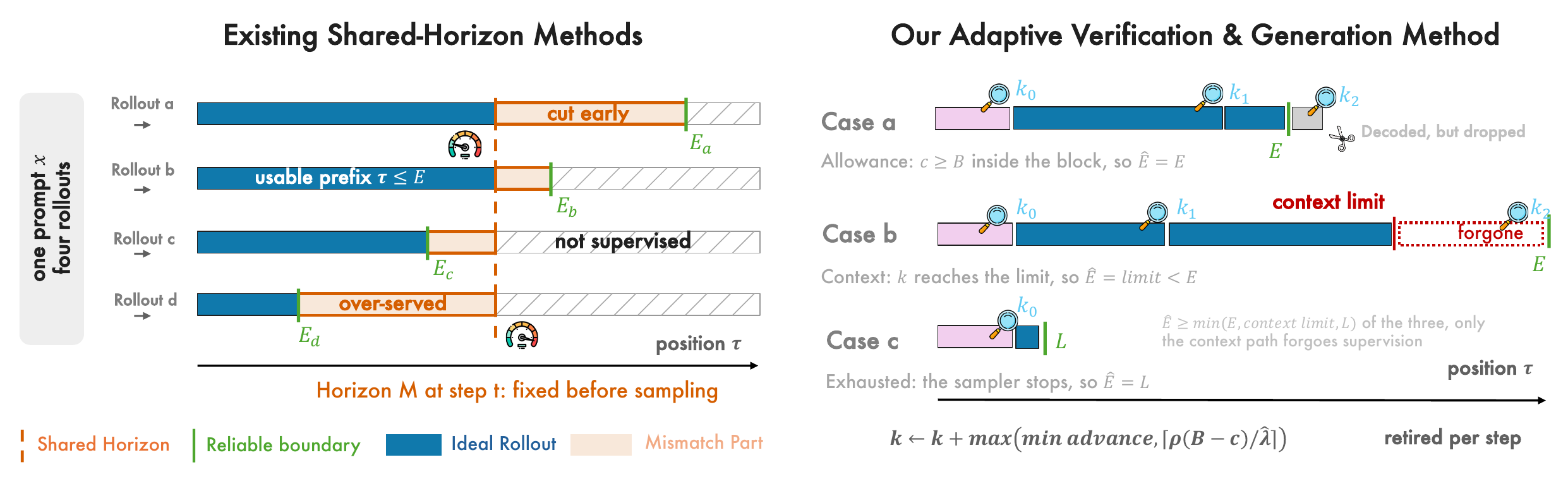}
    \vspace{-5mm}
    \caption{From a shared rollout horizon to adaptive verification. Existing OPD acceleration methods commit to a shared horizon before trajectory sampling, causing premature truncation or redundant generation when usable lengths vary across trajectories; \method{} instead verifies reliability during generation and retires each trajectory at its own boundary, ensuring reliability and efficiency.}
    \label{fig.intro}
    \vspace{-3mm}
\end{figure*}

\begin{itemize}[leftmargin=2.5mm,parsep=3pt]
    \item We identify a granularity mismatch between batch-level rollout horizons and trajectory-specific supervision compatibility, motivating per-sample stopping for long-horizon OPD acceleration.
    \item We propose \method{}, which interleaves cached generation with adaptive verification to determine each trajectory's rollout length from its observed compatibility events.
    \item Extensive experimental results across diverse datasets and teacher--student settings demonstrate that \method{} achieves strong effectiveness, efficiency, robustness, and scalability.
\end{itemize}

%% file: contents/2_RelatedWork.tex
\section{Related Work}

\paragraph{On-Policy Distillation.}
OPD alleviates distribution mismatch by evaluating a frozen teacher on trajectories sampled from the current student. MiniLLM~\citep{gu2024minillm} formalized this paradigm under reverse KL, while GKD~\citep{agarwal2024policy} unified the on-policy / off-policy spectrum across divergences and demonstrated the benefits of on-policy supervision. Subsequent token-level formulations~\citep{lu2025onpolicy} showed that OPD can achieve strong reasoning gains at substantially lower cost than outcome-reward RL, leading to its adoption in production post-training pipelines~\citep{yang2025qwen3,zeng2026glm5} and extensions to multi-teacher and self-distillation settings~\citep{shenfeld2026selfdistillation,hubotter2026reinforcement,ye2026online,zhao2026selfdistilled}. \textit{However, because supervision is obtained along the student's current trajectories, long-horizon rollout remains a major computational bottleneck, even when later tokens provide limited useful signal. This makes rollout horizon a natural target for efficiency optimization.}

\paragraph{Rollout-Horizon Control.}
A growing body of work therefore treats rollout length as an explicit design variable. Open-loop methods determine the horizon from the training schedule: POPD progressively expands it, TOPD permanently truncates it to a fixed fraction~\citep{zhang2026full}, and prefix distillation trains on short prefixes with a linear ramp~\citep{zhang2026fast}. Closed-loop methods instead use runtime signals: \textsc{Prune-OPD} derives a drift signal from token-level top-$k$ overlap to guide dynamic truncation~\citep{yang2026prune}, while \textsc{ADWIN} selects the shortest window whose prefix gradient remains aligned with the full-rollout update~\citep{liang2026adwin}. Despite their different criteria, these methods make horizon decisions at the batch level, forcing all sequences to share the same rollout length. We argue that this granularity is mismatched with the inherently sample-dependent reliability of teacher supervision. PG-OPD~\citep{zhao2026prefix} moves to the sample level by probing fixed prefixes and retaining only high-scoring samples, but yields a binary keep-or-drop decision. \textit{In contrast, we infer a continuous per-sample rollout budget within the existing batch-level horizon, enabling each trajectory to terminate according to its own supervision reliability.}

Table ~\ref{tab:comparison} compares the structural differences between our and other rollout control acceleration methods. Furthermore, we provide a more detailed introduction to additional related work in Appendix~\ref{appendix_related}.

\input{tables/compare}

%% file: tables/compare.tex
\begin{table}[t!]
\centering
\small
\setlength{\tabcolsep}{5pt}
\renewcommand{\arraystretch}{1.2}
\resizebox{1.0\linewidth}{!}{
\begin{tabular}{@{}llllcc@{}}
\toprule
& \multicolumn{2}{c}{\textbf{Length decision}} & \multicolumn{3}{c}{\textbf{Allocation}} \\
\cmidrule(lr){2-3}\cmidrule(lr){4-6}
\textbf{Method} & Driven by & Granularity & Spectrum & \makecell{All seqs.\\trained} & \makecell{Stacks on a\\batch ctrl.} \\
\midrule
\addlinespace[2pt]
POPD~\citep{zhang2026full}            & training step       & batch     & single       & \cmark & \xmark \\
TOPD~\citep{zhang2026full}            & training step       & batch     & single       & \cmark & \xmark \\
Prefix OPD~\citep{zhang2026fast}      & training step       & batch     & single       & \cmark & \xmark \\
\addlinespace[2pt]
\textsc{ADWIN}~\citep{liang2026adwin} & gradient cosine     & batch     & single       & \cmark & \xmark \\
\textsc{Prune-OPD}~\citep{yang2026prune} & top-$k$ overlap  & batch     & single       & \cmark & \xmark \\
\addlinespace[2pt]
PG-OPD~\citep{zhao2026prefix}         & probe overlap       & candidate & two-valued   & \xmark & \cmark \\
\rowcolor{gray!12}
\textbf{\method~~(ours)}                   & usable horizon      & sequence  & continuous   & \cmark & \cmark \\
\bottomrule
\end{tabular}
}
\caption{
{Comparison of rollout-control methods for OPD.}
\textit{Driven by}: what the length decision consumes.
\textit{Granularity}: the unit the decision is made per.
\textit{Spectrum}: the set of realized lengths within one step.
\textit{All seqs.\ trained}: whether every sampled sequence contributes to the loss.
\textit{Stacks on a batch ctrl.}: can run under an independent batch-level window controller instead of replacing it.
}
\label{tab:comparison}
\end{table}

%% file: contents/3_Preliminaries.tex
\section{Preliminaries}
\label{sec:prelim}

\subsection{On-Policy Distillation}

Let $x$ be an input prompt and $y = (y_1,\dots,y_L)$ a response sampled autoregressively from the student policy $\pi_\theta$. At position $\tau$ write $y_{<\tau}$ for the prefix and define
\begin{equation}
p_\tau(v) \triangleq \pi_\theta(v \mid x, y_{<\tau}),
\qquad
q_\tau(v) \triangleq \pi_T(v \mid x, y_{<\tau}),
\label{eq:dists}
\end{equation}
over the vocabulary $v \in \mathcal{V}$, with $\pi_T$ a frozen teacher model. OPD minimizes a token-level divergence on student-sampled trajectories,
\begin{equation}
\mathcal{L}_{\mathrm{OPD}}(\theta) = \mathbb{E}_{x,\,y \sim \pi_\theta(\cdot \mid x)}\Big[\textstyle\sum_{\tau=1}^{L} D_{\mathrm{KL}}\big(p_\tau \,\|\, q_\tau\big)\Big],
\label{eq:opd}
\end{equation}
approximated in practice on the student's top-$k$ candidates to give a per-position reward vector $r_{\tau,\cdot}$.

Eq.~\eqref{eq:opd} exhibits two useful properties. \textit{Locality:} $r_{\tau,\cdot}$ depends only on $(x,y_{\le\tau})$, not on later tokens. \textit{Permanence:} once the prefix is fixed, $r_{\tau,\cdot}$ remains unchanged as generation continues. Thus, a trajectory can be scored incrementally, and stopping at $\tau$ forfeits only subsequent positions without affecting the supervision already obtained. We relegate the proofs of these properties to Appendix~\ref{appendix_theory}.

\subsection{Reliability and Usable Length}

Locality and persistence allow supervision to accumulate gradually, but this does not mean that it maintains equal utility throughout the entire unfolding process. Prior empirical analysis~\citep{li2026rethinking,yang2026prune} links successful OPD to progressive alignment on shared high-probability tokens, showing that optimizing only these tokens recovers nearly all the gains of standard OPD. This finding motivates us to use top-$k$ overlap as a proxy for local teacher--student compatibility. Formally, let $K^S_\tau$ and $K^T_\tau$ denote the top-$k$ sets, and define a low-compatibility event as follow:
\begin{equation}
\phi_\tau
=
\mathds{1}\big[\,|K^S_\tau \cap K^T_\tau|/k < \eta\,\big],
\label{eq:phi}
\end{equation}
where $\eta \in (0,1)$ is the compatibility threshold. To tolerate isolated mismatches, we accumulate these events and define the reliability boundary as a count budget $B \in \mathbb{N}_{>0}$ for such events:
\begin{equation}
c_\tau = \textstyle\sum_{i=1}^{\tau}\phi_i,
\qquad
E = \min\{\tau \in \{1,\dots,L\}: c_\tau \ge B\},
\label{eq:usable}
\end{equation}
where $E$ is the \textit{usable length} of trajectory. Notably, we can naturally observe three crucial properties of $E$: \ding{182} it is \textit{trajectory-specific}; \ding{183} it is the \textit{exact first-passage time} of an integer-valued process rather than a quantity requiring estimation; \ding{184} it is \textit{posterior}, hence unknowable prior to sampling $y$.

%% file: contents/4_Methodology.tex
\section{Methodology}
\label{sec:method}

\subsection{From Horizon Control to Boundary Verification}
\label{sec:method-motivation}

\textbf{Motivation.}
Existing OPD acceleration methods reduce training cost by shortening rollouts. However, as discussed above, Property \ding{182} of reliability length, \ie, \textit{trajectory specificity}, implies that different trajectories have different reliable lengths. For methods that share a rollout horizon across a batch, whether the horizon is scheduled by training progress or adapted from historical feedback, a single scalar is unlikely to fit these heterogeneous boundaries. We therefore argue that the key lies not in adapting a shared horizon, but in providing each trajectory with an appropriate rollout length.

\begin{figure*}[htbp!]
    \centering
    \vspace{-2mm}
    \includegraphics[width=1.0\textwidth]{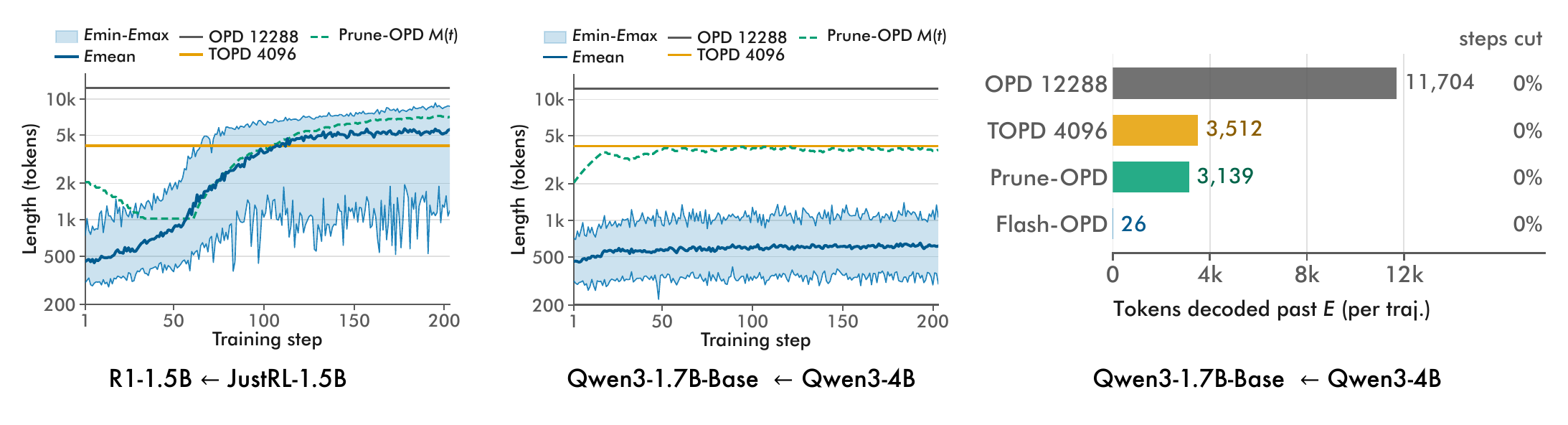}
    \vspace{-6mm}
    \caption{Heterogeneous and evolving reliability lengths expose the mismatch of fixed and adaptive batch-shared horizons (left and middle), while our method \method{} substantially reduces excess rollout budget without prematurely truncating reliable prefixes in the illustrated setting (right).}
    \label{fig.motivation}
\end{figure*}

\textbf{Empirical Study.}
To quantify this phenomenon, we analyze training dynamics on two representative student--teacher pairs, comparing the rollout horizons of existing acceleration methods against the empirical reliability lengths of individual trajectories. The left and middle panels of Figure~\ref{fig.motivation} reveal heterogeneous and dynamically evolving reliability boundaries that neither open-loop nor closed-loop shared-horizon methods can consistently track. In the left panel, TOPD's fixed horizon initially exceeds the entire reliability range but later falls below its mean, while Prune-OPD over-allocates rollout budget early in training. In the middle panel, both methods remain above $E_{\max}$, avoiding premature truncation at the cost of substantial computation beyond the reliable length. The right panel further confirms significant over-decoding beyond reliable lengths across existing horizon-control methods. 
These findings collectively reveal the inefficiency of existing methods.

\textbf{Discussion.}
The above observations suggest that the granularity mismatch between shared horizon and per-trajectory boundaries cannot be resolved solely by a more sensitive horizon controller. A natural solution is to transform budget allocation into per-trajectory stopping decisions, allowing each trajectory to terminate independently based on its own compatibility events. However, the difficulty lies in the property~\ding{184} of $E$: it can only be known after the trajectory is sampled, yet the budget must be set before sampling. Nevertheless, property~\ding{183} offers a promising perspective: $E$ is the first-passage time that depends on prefix counting. Therefore, instead of predicting the $E$ of different trajectories in advance before generation, we can verify whether $c_\tau$ has reached $B$ during trajectory unfolding, thereby naturally deriving a credible $E$. \textit{Locality} and \textit{Permanence} ensure that this does not alter the supervision already computed on the prefix. This motivates us to move the stopping decision inside the generation process, shifting from horizon control to boundary verification.

\subsection{Flash-OPD: Adaptive Boundary Verification}
\label{sec:method-ict}

Building on theses insights, we propose \method{} to robustly accelerate OPD via adaptive boundary verification. Rather than fixing unrolling steps pre-sampling, \method{} dynamically stops each trajectory as it unrolls. Alternating between \textit{generation} and \textit{verification}, rollouts extend to verification points where cumulative counts are evaluated to either halt or continue. The exact usable length determines termination, while the estimated event rate sets the next verification point. 

\begin{algorithm}[htbp]
\caption{\method{}: Fast OPD via Adaptive Verification}
\label{alg:ict}
\begin{algorithmic}[1]
\Require prompts $\mathcal{D}$, student $\pi_\theta$, teacher $\pi_T$,
allowance $B$, initial verification position $k_0$, discount $\rho$,
window $W$, minimum advance $k_{\min}$
\For{training step $t = 1,2,\dots$}
    \State sample prompts
    $\{x^{(b)}\}_{b=1}^{B_{\mathrm{sz}}} \sim \mathcal{D}$;
    set $k^{(b)} \leftarrow k_0$, $c^{(b)} \leftarrow 0$,
    mark all $b$ active
    \While{any $b$ is active}
        \State decode every active $b$ from its current position to
        $k^{(b)}$, resuming from cached prefixes
        \State evaluate $\pi_T$ on the newly produced positions;
        update $c^{(b)} \leftarrow c^{(b)}_{k^{(b)}}$
        by Eqs.~\eqref{eq:phi}--\eqref{eq:usable}
        \For{each active $b$}
            \If{$c^{(b)} \ge B$ \textbf{or} $k^{(b)} \ge L_{\max}$}
                \State $\hat{E}^{(b)} \leftarrow k^{(b)}$;
                mark $b$ inactive
                \Comment{stops at its own boundary}
            \Else
                \State $k^{(b)} \leftarrow k^{(b)}
                + \max\big(k_{\min},
                \lceil \rho (B - c^{(b)})/\hat\lambda^{(b)}
                \rceil\big)$
                \Comment{Eq.~\eqref{eq:step}}
            \EndIf
        \EndFor
    \EndWhile
    \State update $\theta$ on
    $\{y^{(b)}_{\le \hat{E}^{(b)}}\}$
    with rewards $r_{\tau,\cdot}$
\EndFor
\end{algorithmic}
\end{algorithm}

\textbf{Implementation Details.}
\label{sec:method-proc}
We outline the overall procedure in Algorithm~\ref{alg:ict}. Specifically, given an allowance $B$, an initial verification position $k_0$, a discount factor $\rho\in(0,1]$, a trailing window $W$, and a minimum advancement step $k_{\min}$. Each trajectory maintains a verification position $k\leftarrow k_0$ and an accumulated count $c\leftarrow0$, and repeats the following three steps:

\textit{(i) Generate to the verification position.}
Decode from the current position to $k$, resuming from the cached prefix rather than regenerating existing tokens.

\textit{(ii) Verify the accumulated count.}
Evaluate the teacher on the newly produced positions and update $c\leftarrow c_k$ using Eqs.~\eqref{eq:phi}--\eqref{eq:usable}. The resulting $c_k$ is the exact count of observed low-compatibility events.

\textit{(iii) Stop or advance.}
If $c\ge B$ or $k\ge L_{\max}$, set $\hat E=k$ and stop the trajectory. Otherwise, advance the verification position according to
\begin{equation}
k \;\leftarrow\; k + \max\!\Big(k_{\min},\;
\big\lceil \rho\,(B - c)/\hat\lambda_k \big\rceil\Big),
\qquad
\hat\lambda_k = \big(c_k - c_{k-W}\big)/W .
\label{eq:step}
\end{equation}
Here, \(\hat\lambda_k\) estimates the local rate of low-compatibility events over the latest \(W\) positions. With \(B-c\) events remaining before the boundary, \((B-c)/\hat\lambda_k\) estimates how many additional tokens would be needed if the current rate persisted. We advances only a fraction \(\rho\) of this estimated distance before measuring the rate again, yielding large steps far from the boundary and progressively smaller steps as the boundary approaches. The minimum step \(k_{\min}\) ensures continued progress, while the actual stopping decision always relies on the exact count \(c_k\ge B\). Figure~\ref{fig.example} provides an intuitive example.

\begin{figure*}[t!]
    \centering
    \includegraphics[width=1.0\textwidth]{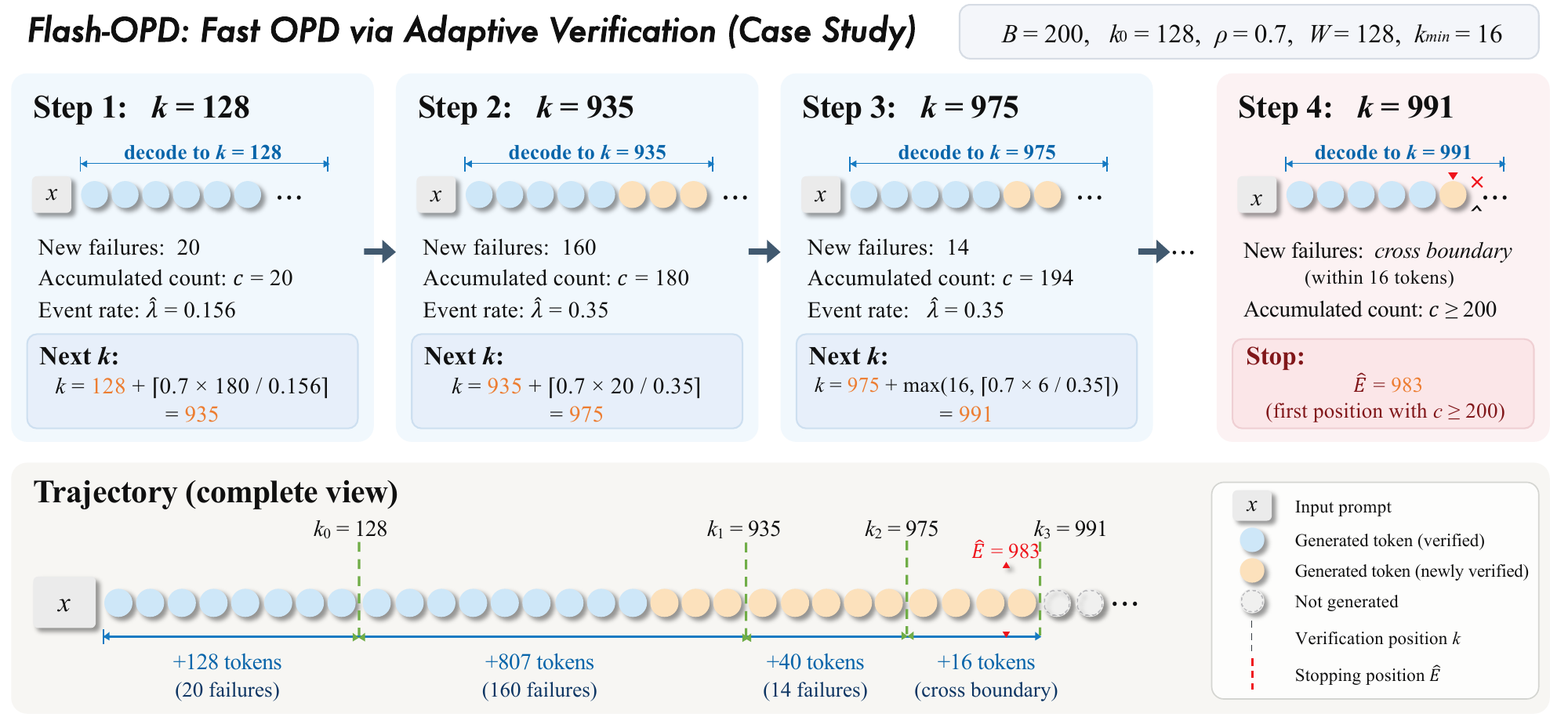}
    \vspace{-4mm}
    \caption{A schematic diagram of our rollout case, which adaptively adjusts verification positions according to the estimated local event rate and stops once the exact reliability boundary is reached.}
    \vspace{-5mm}
    \label{fig.example}
\end{figure*}

\textbf{Theoretical Analysis.}
\label{sec:method-guarantee}
A potential concern is that an improper final rate estimate may compromise supervision. We show that extrapolation error only cause a slight decoding cost, not supervision:

\begin{proposition}
\label{prop:safe}
$\hat{E}^{(b)} \ge E^{(b)}$ for every trajectory, for any $\rho \in (0,1]$, any $W$, any realization of $\hat\lambda$.
\end{proposition}
\begin{proof}
Termination occurs only under $c_k \ge B$, evaluated on the exact count. Since $c_\tau$ is non-decreasing and $E^{(b)} = \min\{\tau : c_\tau \ge B\}$, this cannot hold at any $k < E^{(b)}$. The estimate $\hat\lambda$ enters only the placement of subsequent verification boundaries, never the termination test.
\end{proof}


\textbf{Complexity Analysis.}
\label{sec:method-complexity}
\method{} replaces a long generation pass with $R$ shorter ones, so the key questions are how large $R$ can be and how much each round costs. Let $R$ denote the number of verification rounds. Since each advance in Eq.~\eqref{eq:step} is at least $k_{\min}$ and $k\le L_{\max}$, we have $R \le 1+\lceil(L_{\max}-k_0)/k_{\min}\rceil$. When the local event rate is approximately stable, each round consumes a fraction $\rho$ of the remaining allowance, yielding $R=O(\log B)$. The per-step cost consists mainly of decoding and teacher evaluation, and our method do not incur additional overhead beyond these. Importantly, with prefix caching, each token is decoded only once, so the generation cost remains proportional to the actual rollout length. Compared with the full rollout length $L_{max}$, \method{} substantially reduces this cost by terminating generation around the usable length $E$.

%% file: contents/5_Experiments.tex
\section{Experiments}
\label{sec:experiments}

In this section, we conduct extensive experiments to investigate the following research questions:

\begin{itemize}[leftmargin=2.5mm,parsep=3pt]
    \item \textbf{RQ1:} Does \method{} consistently improve the accuracy--cost trade-off? \textit{(\textbf{Effectiveness})}
    \item \textbf{RQ2:} Where exactly do the wall-clock savings of \method{} come from? \textit{(\textbf{Efficiency})}
    \item \textbf{RQ3:} How does \method{} improve training dynamics through adaptive rollout? \textit{(\textbf{Mechanism})}
    \item \textbf{RQ4:} Do the gains of \method{} transfer reliably out of distribution? \textit{(\textbf{Transferability})}
    \item \textbf{RQ5:} How do the efficiency gains of \method{} scale with larger teacher size? \textit{(\textbf{Scalability})}
    \item \textbf{RQ6:} How does \method{} work, and does hyperparameter tuning affect it? \textit{(\textbf{Robustness})}
\end{itemize}

\subsection{Experimental Setup}
\label{sec:setup}

\noindent\textbf{Datasets and Baselines.}
We compare vanilla OPD~\citep{lu2025onpolicy}, and representative methods for open-loop and closed-loop control: TOPD~\citep{zhang2026full}, and Prune-OPD~\citep{yang2026prune}, all trained on DAPO-Math-17K~\citep{yu2025dapo}. Evaluation covers distribution-aligned math benchmarks (AMC23, AIME24/25, HMMT24/25) and transfer benchmarks (MATH-500, Minerva Math, Olympiad-Bench, GPQA), with GPQA testing beyond mathematics. Details are in Appx.~\ref{appendix_datasets}.

\noindent\textbf{Models and Protocol.}
To assess generality across representative OPD settings, we pair DeepSeek-R1-Distill-Qwen-1.5B with JustRL-DeepSeek-1.5B (same-scale, post-training-enhanced teacher) or DeepSeek-R1-Distill-Qwen-7B (larger teacher, similar lineage), and Qwen3-1.7B/4B-Base with Qwen3-4B (Non-thinking), varying student scale and Base-to-post-trained distribution gaps. All methods train for 203 steps under identical optimization settings. We report mean@16 averaged over the final three steps (31,744-token decoding budget) and efficiency metrics, including training time. Baselines use recommended hyperparameters, while we use a fixed configuration unless otherwise specified: $\eta=0.7$, $B=200$, $\rho=1.0$, and $k_0=W=128$. For more details on methods, evaluation protocols, parameter settings, training hardware, etc., please refer to Appendix~\ref{appendix_protocol}.

\subsection{Effectiveness Study (RQ1)}

Table~\ref{tab:main-results} compares \method{} with vanilla OPD and advanced OPD acceleration methods across four diverse student--teacher pairs and five in-domain benchmarks. We make three observations:

\input{tables/overall_performance}

\ding{182} \method{} consistently improves the accuracy--cost trade-off across model pairs. Compared with vanilla OPD, it achieves $2.23\times$--$7.52\times$ speedups while largely preserving, and in most settings even improving, average performance.
\ding{183} Interestingly, on Qwen3-4B-Base / Qwen3-4B, \method{} delivers not only substantial acceleration but also a notable accuracy gain. We find that the student in this setting is prone to uncontrolled response-length growth during on-policy distillation, where overly long rollouts can introduce low-quality or even harmful supervision. In contrast, \method{} truncates trajectories once the supervision signal becomes unreliable. This adaptive filtering may therefore act as an implicit curriculum and regularizer, leading to better optimization.
\ding{184} Compared with existing acceleration methods, \method{} is the fastest in all four settings while maintaining competitive accuracy. The advantage is pronounced on the Qwen family, where it achieves $7.14\times$ and $7.52\times$ speedups, substantially outperforming existing baselines.





\subsection{Efficiency Study (RQ2)}

To understand where the efficiency gains come from, Figure~\ref{fig.steptime} decomposes the mean wall-clock time per training step into rollout generation, log-probability computation, teacher verification, policy updates, and residual overhead across the four model pairs. We make three observations:\\
\begin{wrapfigure}{c}{11.5cm}
    \begin{center}
    \vspace{-6mm}
    \includegraphics[width=1.0\linewidth]{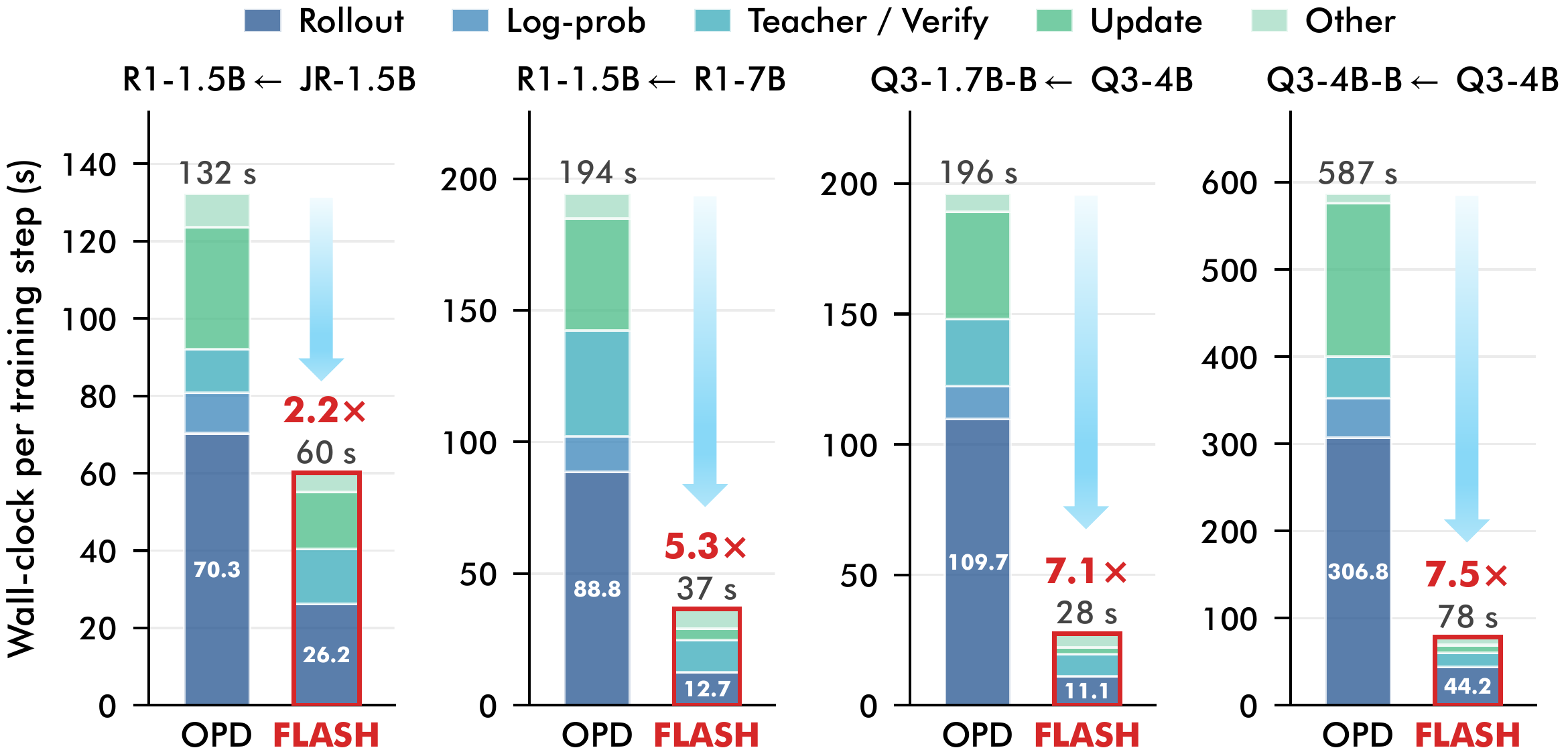}
    \vspace{-6mm}
    \caption{Mean wall-clock seconds per training step.}
    \label{fig.steptime}
    \vspace{-10mm}
    \end{center}
\end{wrapfigure}
\ding{182} \method{} substantially reduces per-step training time, primarily by shortening rollout generation. \ding{183} \method{} evaluates both student and teacher during verification and reuses the results, eliminating the separate log-probability pass. \ding{184} By terminating rollouts early, \method{} further reduces the costly policy-update time, extending generation savings to subsequent stages.

\subsection{Mechanism Study (RQ3)}

We further investigate the mechanisms of our method compared to vanilla OPD through training dynamics. Figure~\ref{fig.frontier} compares the relationship between accuracy and wall-clock training time for \method{} relative to OPD across four model pairs. We observe the following phenomena:

\begin{figure}[htbp!]
    \centering
    \includegraphics[width=1.0\textwidth]{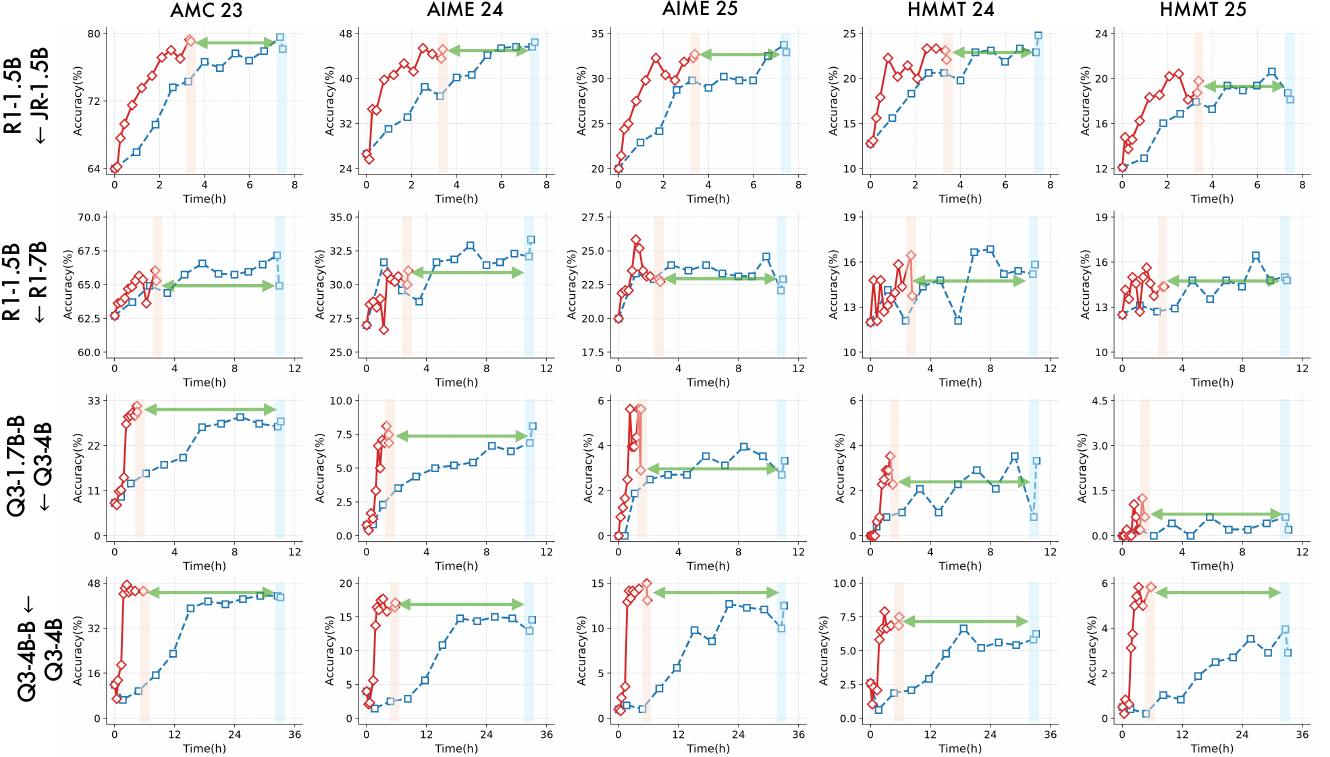}
    \caption{Wall-clock training dynamics of vanilla OPD and our \method{} across four teacher-student model pairs and five benchmarks. We perform a benchmark evaluation every 20 steps.}
    \label{fig.frontier}
\end{figure}

\ding{182} \method{} consistently shifts the learning curves toward the upper left: across model pairs and benchmarks, it reaches the accuracy regime of OPD substantially earlier, indicating faster optimization progress per unit wall-clock time.
\ding{183} On the DeepSeek-family pairs, \method{} largely preserves the optimization trajectory of OPD in one setting while compressing it in time. In the other, both trajectories are relatively noisy and yield limited gains, consistent with prior observations~\citep{li2026rethinking} that a teacher from the same training pipeline may not provide sufficiently novel and transferable supervision on the states visited by the student.
\ding{184} On the Qwen-family pairs, we surprisingly observe a sharp early improvement with \method{}, whereas vanilla OPD improves much more gradually over a substantially longer horizon. This is most evident for Qwen3-4B-Base $\leftarrow$ Qwen3-4B, where \method{} not only reaches strong performance within a few hours, but also converges to a higher-accuracy regime on all five benchmarks. This suggests that adaptive rollout truncation can improve optimization itself, potentially by filtering out low-value supervision in later portions of the trajectory.
\ding{185} OPD often continues to consume substantial wall-clock time after entering a plateau, yielding only marginal gains or oscillations. In contrast, \method{} concentrates computation in the earlier, more productive stages of training, suggesting that shorter, higher-confidence supervision is often easier for the current policy to absorb. By prioritizing tokens that the policy can effectively exploit, rather than forcing it to match the teacher's full long trajectory at once, \method{} may reduce gradient variance and error accumulation, leading to more efficient optimization.


\subsection{Transferability Study (RQ4)}

To examine whether the benefits of \method{} extend beyond in-domain evaluation and same-lineag teacher--student pairs, we study transferability along two dimensions: evaluation distribution and model generation. Figure~\ref{fig.transfer} summarizes the results. We make three observations:

\begin{figure}[htbp!]
    \centering
    \includegraphics[width=1.0\textwidth]{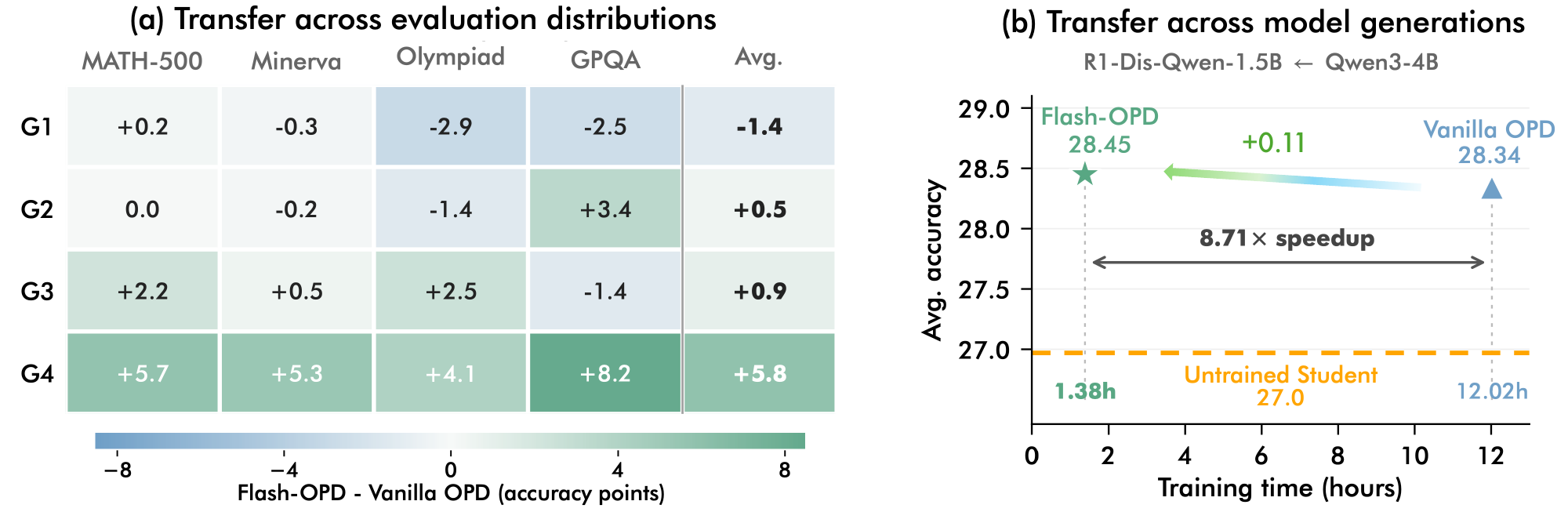}
    \caption{Transferability of \method{} across evaluation distributions and model generations.}
    \label{fig.transfer}
\end{figure}

\ding{182} Across unseen evaluation distributions, \method{} largely preserves OPD performance: the macro differences on G1--G3 remain within $1.4$ points, despite including GPQA, a non-mathematical benchmark.
\ding{183} Transfer can also be favorable: on G4, \method{} consistently improves all four held-out benchmarks, yielding a $+5.84$-point macro gain.
\ding{184} The efficiency gain persists across model families without pair-specific tuning: on R1-1.5B$\leftarrow$Qwen3-4B, \method{} matches OPD in macro accuracy ($28.45$ vs.\ $28.34$) while reducing training time from $12.02$h to $1.38$h ($8.71\times$).

\subsection{Scalability Study (RQ5)}


To examine whether \method{} remains effective with increasingly large teachers, Figure~\ref{fig.scalability} scales the teacher from Qwen3-4B to 14B and 32B, fixing the Qwen3-1.7B-Base student. We observe that:

\begin{figure}[htbp!]
    \centering
    \includegraphics[width=1.0\textwidth]{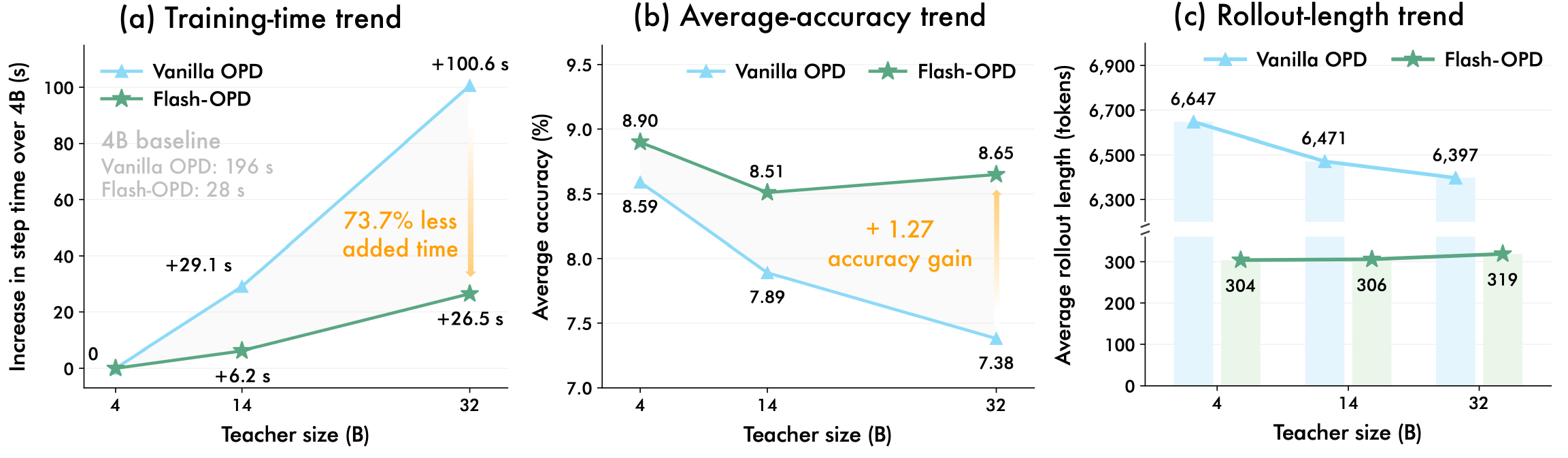}
    \caption{With a fixed student, \method{} incurs smaller step-time increases than OPD while maintaining stable accuracy and consistently short rollouts as teacher size grows from 4B to 32B.}
    \label{fig.scalability}
\end{figure}

\ding{182} \method{} remains substantially faster: from 4B to 32B, per-step time rises by only $26.5$\,s \textit{vs.} $100.6$\,s for vanilla OPD, a $73.7\%$ reduction in added time.
\ding{183} Accuracy remains stable at $8.5$--$8.9\%$, while OPD drops from $8.59\%$ to $7.38\%$, giving \method{} a $1.27$ percentage-point advantage at 32B.
\ding{184} Larger teachers do not consistently improve OPD, whose accuracy varies with rollout length. We maintains nearly constant rollout lengths, yielding more stable training across teacher scales.

\subsection{Robustness Study (RQ6)}


To understand how controller settings affect OPD behavior, Figure~\ref{fig.robustness} examines the joint variation of the $\eta$ and $B$, as well as the impact of $\rho$ on validation. We can observe that:

\begin{figure}[htbp!]
    \centering
    \includegraphics[width=1.0\textwidth]{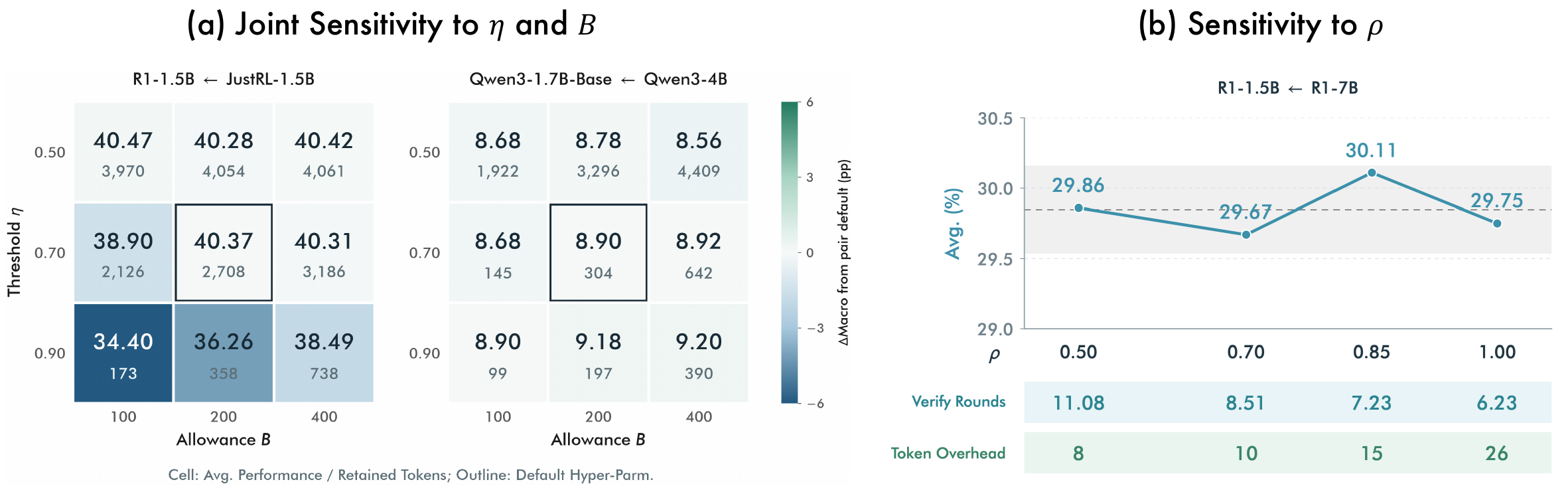}
    \caption{Sensitivity of \method{} to the drift threshold $\eta$, allowance $B$, and step discount $\rho$.}
    \label{fig.robustness}
\end{figure}

\ding{182} \method{} remains robust to broad variations in controller parameters, with degradation mainly at extreme thresholds that aggressively shift stopping boundaries.
\ding{183} The threshold and allowance interact: larger $B$ partly mitigates degradation at high $\eta$ but has little effect at low $\eta$.
\ding{184} The step discount controls the verification trade-off; however, as proof shown earlier, the growth of redundant token generation is quite limited, and overall performance remains relatively stable.

%% file: tables/overall_performance.tex
\begin{table*}[htbp!]
    \vspace{-2mm}
    \centering
    \small
    \renewcommand\arraystretch{1.15}
    \setlength{\tabcolsep}{8pt}
    \caption{Overall performance-efficiency comparison under different benchmarks and settings.}
    \vspace{-2mm}
    \label{tab:main-results}
    \resizebox{1.0\linewidth}{!}{
    \begin{tabular}{l|c|c|c|c|c|c|c|c}
       \toprule
       \textbf{Method} & \textbf{AMC23} & \textbf{AIME24} & \textbf{AIME25} & \textbf{HMMT24} & \textbf{HMMT25} & \textbf{Average} & \textbf{Time} & \textbf{Speed Up} \\
       \midrule
       \rowcolor{gray!10} \multicolumn{9}{c}{\textit{Student / Teacher: DeepSeek-R1-Distill-Qwen-1.5B / JustRL-DeepSeek-1.5B}} \\
       OPD & 78.16 & 46.46 & 32.92 & 24.79 & 18.12 & 40.09 & 7.46h & - \\
       TOPD & 79.67 & 45.83 & 33.33 & 24.79 & 19.38 & \textbf{40.60} & 3.51h & 2.12$\times$ \\
       Prune-OPD & 78.92 & 44.79 & 31.25 & 22.92 & 20.62 & 39.70 & \underline{3.36h} & 2.22$\times$ \\
       \rowcolor{cyan!8} \textbf{\method} & 79.37 &  46.46 &  32.92 &  22.29 &  20.83 & \underline{40.37} & \textbf{3.35h} & 2.23$\times$ \\
       \midrule
       \rowcolor{gray!10} \multicolumn{9}{c}{\textit{Student / Teacher: DeepSeek-R1-Distill-Qwen-1.5B / DeepSeek-R1-Distill-Qwen-7B}} \\
       OPD & 64.91 & 33.33 & 22.92 & 15.83 & 14.79 & \underline{30.36} & 10.95h & - \\
       TOPD & 64.91 & 32.71 & 23.75 & 16.25 & 15.00 & \textbf{30.52} & 4.15h & 2.64$\times$ \\
       Prune-OPD & 65.36 & 30.83 & 24.38 & 13.75 & 15.42 & 29.95 & \underline{3.81h} & 2.88$\times$ \\
       \rowcolor{cyan!8} \textbf{\method} & 66.27 & 30.42 & 24.17 & 13.33 & 14.58 & 29.75 & \textbf{2.06h} & 5.32$\times$\\
       \midrule
       \rowcolor{gray!10} \multicolumn{9}{c}{\textit{Student / Teacher: Qwen3-1.7B-Base / Qwen3-4B (Non-thinking)}} \\
       OPD & 27.94 & 8.12 & 3.33 & 3.33 & 0.21 & \underline{8.59} & 11.06h & - \\
       TOPD & 28.35 & 7.85 & 3.13 & 2.15 & 0.28 & 8.35 & 4.80h & 2.30$\times$ \\
       Prune-OPD & 28.43 & 6.74 & 3.89 & 2.36 & 0.49 & 8.38 & \underline{3.04h} & 3.64$\times$ \\
       \rowcolor{cyan!8} \textbf{\method} & 29.29 &   6.46 &   5.00 &   3.33 &   0.42 &  \textbf{8.90} & \textbf{1.55h} & 7.14$\times$ \\
       \midrule
       \rowcolor{gray!10} \multicolumn{9}{c}{\textit{Student / Teacher: Qwen3-4B-Base / Qwen3-4B (Non-thinking)}} \\
       OPD & 43.00 & 14.58 & 12.50 & 6.25 & 2.92 & \underline{15.85} & 33.09h & - \\
       TOPD & 42.09 & 15.00 & 12.50 & 5.21 & 3.12 & 15.58 & \underline{7.81h} & 4.24$\times$ \\
       Prune-OPD & 40.06 & 15.42 & 10.83 & 4.79 & 2.71 & 14.76 & 13.87h & 2.39$\times$ \\
       \rowcolor{cyan!8} \textbf{\method} & 48.72 &  18.75 &  15.83 &   7.08 &   5.62 & \textbf{19.20}  &  \textbf{4.40h} & 7.52$\times$ \\
       \bottomrule
    \end{tabular}
    }
    \vspace{-2mm}
\end{table*}

%% file: contents/6_Conclusion.tex
\section{Conclusion}
\label{sec:conclusion}

In this paper, we investigate the mismatch between shared rollout horizons and trajectory-specific reliability boundaries in on-policy distillation. We propose \method{}, an adaptive boundary verification framework that alternates generation and verification, using exact compatibility-event counts to guide stopping decisions and local event rates to schedule verification. Extensive experiments on multiple teacher-student model pairs demonstrate its effectiveness, efficiency, transferability, scalability, and robustness. In future work, we aim to extend this efficient acceleration method for adaptive boundary verification to multi-teacher online policy distillation and agent tasks.

\clearpage

%% file: appendix/0_appendix_list.tex
\onecolumn
\renewcommand{\cftsecfont}{\bfseries}


\begin{center}
    \Large{\sc\Huge Supplementary material\\\large ~\\Flash-OPD: Fast On-Policy Distillation}\\
\end{center}
\vskip 4mm
\startcontents[sections]\vbox{\sc\Large Table of Contents}
\vspace{5mm}
\hrule height .8pt
\vspace{-2mm}
\printcontents[sections]{l}{1}{\setcounter{tocdepth}{2}}
\vspace{4mm}
\hrule height .8pt
\vskip 10mm


\clearpage
\input{appendix/6_morerelated}
\input{appendix/1_theoretical}
\input{appendix/2_method}

\input{appendix/3_expsetting}

\input{appendix/4_moreresult}
\input{appendix/5_morediscussion}

%% file: appendix/6_morerelated.tex
\section{More Related Work}\label{appendix_related}

\subsection{Distillation Objectives}

Classical knowledge distillation trains a student to match teacher predictions on a fixed corpus~\citep{hinton2015distilling,sanh2019distilbert,jiao2020tinybert,wang2020minilm}. On-policy distillation instead queries the teacher on student-generated prefixes, reducing the mismatch between training states and those visited by the student~\citep{ross2011reduction,gu2024minillm,agarwal2024policy,lu2025onpolicy}. Recent work improves the supervision at each position through reward extrapolation~\citep{yang2026learning}, revised divergence objectives~\citep{jang2026stable,jin2026entropy,ko2026scaling,fu2026revisiting}, and analyses of teacher--student compatibility~\citep{li2026rethinking,li2025small,kim2026does}. Self-distillation extends the source of supervision beyond a fixed external teacher~\citep{shenfeld2026selfdistillation,hubotter2026reinforcement,zhao2026selfdistilled,ye2026online}. These directions change the training signal, whereas \method{} controls how far each trajectory is generated while retaining the token-level objective on its prefix.

\subsection{Decoding and Training Efficiency}

Speculative decoding reduces generation cost by verifying blocks of proposed tokens. Its verification determines which proposed tokens to accept. The verification in \method{} instead determines whether a student trajectory should continue receiving distillation supervision; it does not replace generated tokens with teacher samples. Improvements to decoding or other training stages~\citep{wu2026lightningopd,cai2026learning} address costs complementary to rollout truncation. Methods that train models to reason in fewer tokens at inference time~\citep{accordion,sang2026crispcompressedreasoningiterative} target the deployed model's response length rather than the length of training rollouts.

%% file: appendix/1_theoretical.tex
\section{More Analysis}\label{appendix_theory}

This section proves the prefix properties used in Section~\ref{sec:prelim} and gives details of the stopping guarantee and verification complexity.

\subsection{Locality and Permanence}\label{proof_local}

\begin{lemma}[Locality and Permanence]\label{lem:local}
Fix the student and teacher parameters during rollout collection. At each position $\tau$, the per-position reward $r_{\tau,\cdot}$ depends only on $(x,y_{<\tau})$. Any two trajectories with the same prefix $y_{<\tau}$ therefore receive the same reward at position $\tau$, regardless of their later tokens.
\end{lemma}

\begin{proof}
By causality, $p_\tau=\pi_\theta(\cdot\mid x,y_{<\tau})$ and $q_\tau=\pi_T(\cdot\mid x,y_{<\tau})$ are determined by the prefix. The student's top-$k$ candidates and their weights are functions of $p_\tau$. Both $D_{\mathrm{KL}}(p_\tau\|q_\tau)$ and its top-$k$ approximation are therefore determined by the same prefix. Extending the trajectory changes neither distribution at position $\tau$ nor the resulting reward.
\end{proof}

\begin{corollary}[Prefix supervision]\label{cor:trunc}
For any retained length $S\le L$, the truncated sum $\sum_{\tau=1}^{S}D_{\mathrm{KL}}(p_\tau\|q_\tau)$ equals the first $S$ terms of the full-trajectory sum in Eq.~\eqref{eq:opd}.
\end{corollary}

\begin{proof}
Truncation leaves every prefix $y_{<\tau}$ with $\tau\le S$ unchanged, so Lemma~\ref{lem:local} applies to each retained.
\end{proof}

The implementation averages the loss over retained tokens. Truncation changes this normalization and removes the tail terms, but does not change the per-position rewards on the retained prefix.

\subsection{Stopping Guarantee}\label{proof_safe}

For termination triggered by the allowance, Proposition~\ref{prop:safe} follows from the monotonicity of $c_\tau$: the test $c_k\ge B$ cannot succeed before $E$. The rate estimate affects the next verification position, not this test. The implementation also handles EOS and the length limit as follows.

\begin{proposition}[Stopping with EOS and a length limit]\label{prop:safecap}
Let $L$ be the natural termination length and set $E=\infty$ if the event count never reaches $B$. Algorithm~\ref{alg:ict-impl}, which locates the first crossing inside each verified block, retains
\[
\hat E=\min(E,L_{\max},L).
\]
\end{proposition}

\begin{proof}
If a verified block contains the first crossing, scanning its event counts identifies $E$ and retains exactly $y_{\le E}$. If no crossing occurs before $L_{\max}$, the algorithm retains all $L_{\max}$ generated positions. If the sampler terminates earlier without a crossing, it retains the complete response of length $L$. These cases give the stated minimum.
\end{proof}

\begin{remark}[Verification position and retained length]\label{rem:tighter}
Algorithm~\ref{alg:ict} uses the verification position as the retained length. The implementation in Algorithm~\ref{alg:ict-impl} instead finds the first crossing within the verified block and drops tokens after it. For a fixed trajectory, changing $\rho$ changes the verification positions but not this crossing. If the crossing is found at verification position $k$, the decoding overshoot is $k-E\ge0$.
\end{remark}

\subsection{Verification Complexity}\label{proof_rounds}

Assume $k_0\le L_{\max}$. Each round after the first advances by at least $k_{\min}$, except for a final step clipped to $L_{\max}$. Hence the number of rounds satisfies
\[
R\le 1+\left\lceil\frac{L_{\max}-k_0}{k_{\min}}\right\rceil.
\]
EOS or reaching the allowance can end the loop earlier.

The logarithmic behavior in Section~\ref{sec:method-complexity} describes a stable-rate approximation. Ignoring integer rounding, suppose a block of $m$ tokens adds $\lambda m$ events and $\hat\lambda=\lambda>0$. For residual allowance $a=B-c$, Eq.~\eqref{eq:step} advances $\rho a/\lambda$ tokens when the minimum step is inactive, leaving $(1-\rho)a$. For fixed $0<\rho<1$, the residual falls below one event after $O(\log B)$ rounds. At $\rho=1$, this approximation reaches the boundary in one additional round. Actual round counts depend on event-rate variation; the finite-length bound above holds independently of that approximation.

%% file: appendix/2_method.tex
\section{Method Details}\label{method_process}

\subsection{Batched Implementation}\label{app:method-alg}

Algorithm~\ref{alg:ict-impl} gives the batched implementation of \method{}, including termination at EOS or the length limit.

\input{tables/algorithm}

\textbf{Active trajectories.} Each round processes only trajectories that have not stopped. Requests are padded to a multiple of the worker count and unpadded on return. Trajectory indices remain fixed during the loop so that each prefix is matched to its own event count; length-based reordering is applied after the loop.

\textbf{Prefix resumption.} Each request extends $x^{(i)}\Vert y^{(i)}$ by at most $\delta^{(i)}$ tokens. Prefix caching avoids regenerating tokens from earlier rounds. Verification computes student and teacher distributions, and the controller reads the event indicators for the newly generated positions.

\textbf{Stopping within a block.} The controller scans the new positions in order and retains the prefix through the first position where $c_\tau\ge B$. If no crossing occurs, it retains the complete response when the sampler terminates, or stops at $L_{\max}$. A zero estimated event rate sends the next verification position to $L_{\max}$.

\subsection{Decoding and Verification Cost}\label{app:method-cost}

Algorithm~\ref{alg:ict-impl} records two token counts. $D$ counts generated positions, each once. $P$ sums the verified prefix lengths across rounds because verification forwards recompute the prefix without a persistent KV cache. Thus decoding cost depends on $D$, while verification cost depends on $P$.

Verification includes both a student forward pass, which supplies the student's top-$k$ candidates, and a teacher forward pass. Table~\ref{tab:mechanism} reports the number of rounds, verification tokens per retained token, and decoding overshoot. Table~\ref{tab:efficiency} gives the corresponding time breakdown.

\subsection{Rollout Engine Settings}\label{app:method-engine}

\textbf{Cache lifetime and memory.} The vLLM engine remains active throughout the verification loop and is put to sleep once all trajectories stop. We set \texttt{gpu\_memory\_utilization} to $0.6$ to leave memory for verification, with a verification micro-batch budget of $4096$ tokens.

\textbf{Reusing verification outputs.} Student log-probabilities and teacher outputs can be cached for each verified span and reused in the subsequent loss computation. Reuse is enabled only when every retained trajectory has complete coverage.

%% file: tables/algorithm.tex
\begin{algorithm}[t!]
\caption{\method: detail implementation version, one training step.}
\label{alg:ict-impl}
\begin{algorithmic}[1]
\Require prompts $\{x^{(i)}\}_{i=1}^{N}$; student $\pi_\theta$; frozen teacher $\pi_T$; allowance $B$; initial checkpoint $k_0$; discount $\rho$; window $W$; minimum advance $k_{\min}$; context limit $L_{\max}$; overlap threshold $\eta$; candidate-set size $k$
\Ensure retained prefixes $\{y^{(i)}_{\le\hat{E}^{(i)}}\}$; cost counters $D$ (decoded positions), $P$ (prefix positions)
\State $\mathcal{A} \leftarrow \{1,\dots,N\}$; \; $D \leftarrow 0$; \; $P \leftarrow 0$ \Comment{$\mathcal{A}$ is the active set}
\For{$i = 1,\dots,N$}
  \State $y^{(i)} \leftarrow (\,)$; \; $v^{(i)} \leftarrow 0$; \; $c^{(i)} \leftarrow 0$; \; $k^{(i)} \leftarrow \min(k_0,\, L_{\max})$ \Comment{$v$: verified positions, $c$: drift count}
\EndFor
\While{$\mathcal{A} \neq \emptyset$}
  \State $\delta^{(i)} \leftarrow k^{(i)} - v^{(i)}$ for every $i \in \mathcal{A}$ \Comment{tokens still owed to this trajectory's checkpoint}
  \State $Y \leftarrow \textsc{Generate}\big(\{x^{(i)} \Vert y^{(i)}\}_{i \in \mathcal{A}}\big)$ with per-request cap $\delta^{(i)}$ \Comment{resume from the prefix cache; retired rows not dispatched}
  \State $m^{(i)} \leftarrow \min\big(|Y^{(i)}|,\; \delta^{(i)}\big)$ \Comment{clamp: $\mathrm{pad}\!=\!\mathrm{eos}$ over-counts by one at the cap}
  \State $\Phi \leftarrow \textsc{Verify}\big(\{x^{(i)} \Vert y^{(i)} \Vert Y^{(i)}\}_{i \in \mathcal{A}}\big)$ \Comment{$\pi_\theta$ pass, then $\pi_T$ pass; outputs read on the new block only}
  \State $\phi^{(i)}_{v^{(i)}+j} \leftarrow \mathds{1}\big[\,|K^S \cap K^T|/k < \eta\,\big]$, \; $j = 1,\dots,m^{(i)}$ \Comment{Eq.~\eqref{eq:phi}}
  \For{$i \in \mathcal{A}$}
    \State $s \leftarrow \varnothing$
    \For{$j = 1,\dots,m^{(i)}$} \Comment{localise the crossing inside the verified block}
      \State $c^{(i)} \leftarrow c^{(i)} + \phi^{(i)}_{v^{(i)}+j}$; \; \textbf{if} $c^{(i)} \ge B$ \textbf{and} $s = \varnothing$ \textbf{then} $s \leftarrow v^{(i)} + j$
    \EndFor
    \State $y^{(i)} \leftarrow y^{(i)} \Vert Y^{(i)}_{1:m^{(i)}}$; \; $v^{(i)} \leftarrow v^{(i)} + m^{(i)}$; \; $D \leftarrow D + m^{(i)}$; \; $P \leftarrow P + v^{(i)}$
    \If{$s \neq \varnothing$}
      \State $\hat{E}^{(i)} \leftarrow s$; \; $\mathcal{A} \leftarrow \mathcal{A} \setminus \{i\}$ \Comment{retire: \textsf{allowance}}
    \ElsIf{$m^{(i)} < \delta^{(i)}$}
      \State $\hat{E}^{(i)} \leftarrow v^{(i)}$; \; $\mathcal{A} \leftarrow \mathcal{A} \setminus \{i\}$ \Comment{retire: \textsf{exhausted} (sampler emitted EOS)}
    \ElsIf{$v^{(i)} \ge L_{\max}$}
      \State $\hat{E}^{(i)} \leftarrow v^{(i)}$; \; $\mathcal{A} \leftarrow \mathcal{A} \setminus \{i\}$ \Comment{retire: \textsf{context}}
    \Else
      \State $W' \leftarrow \min\big(W,\, v^{(i)}\big)$; \quad $\hat\lambda \leftarrow \frac{1}{W'}\sum_{\tau = v^{(i)}-W'+1}^{v^{(i)}} \phi^{(i)}_{\tau}$ \Comment{trailing-window drift rate}
      \State $a \leftarrow L_{\max}$ \textbf{if} $\hat\lambda = 0$, \textbf{else} $a \leftarrow \big\lceil \rho\,(B - c^{(i)})/\hat\lambda \big\rceil$
      \State $k^{(i)} \leftarrow \min\big(v^{(i)} + \max(k_{\min},\, a),\; L_{\max}\big)$ \Comment{Eq.~\eqref{eq:step}}
    \EndIf
  \EndFor
\EndWhile
\State \Return $\{y^{(i)}_{\le\hat{E}^{(i)}}\}_{i=1}^{N}$ and $(D, P)$ \Comment{row order unchanged; length rebalancing runs only now}
\end{algorithmic}
\end{algorithm}
\vspace{-6mm}

%% file: appendix/3_expsetting.tex
\section{Experimental Details}

\subsection{Datasets}\label{appendix_datasets}

Table~\ref{tab.datasets} summarizes the training and evaluation datasets. All methods train on DAPO-Math-17K~\citep{yu2025dapo}, which contains $17{,}917$ competition-mathematics prompts. Each run uses $203$ steps of $64$ prompts, for $12{,}992$ sampled prompts without repetition.

\input{tables/dataset}

The in-domain evaluation contains $203$ problems across AMC23, AIME24, AIME25, HMMT24 and HMMT25. For transfer evaluation, we sample $250$ problems uniformly without replacement from each of MATH-500, Minerva and Olympiad-Bench, and use all GPQA. The resulting $948$ prompts are fixed across methods. Each response has a decoding budget of $31{,}744$ tokens, with $16$ samples per prompt in-domain and $4$ for transfer. GPQA evaluates generalization beyond mathematics.

\subsection{Teacher--Student Pairs}\label{appendix_pairs}

\input{tables/pairs}

We refer to the four pairs in Table~\ref{tab:pairs}, in order, as G1--G4. All use $\eta=0.7$, $B=200$ and $\rho=1.0$. G2--G4 use $k_0=W=128$ and $k_{\min}=16$. G1 uses $k_0=W=1024$, $k_{\min}=512$, and $L_{\max}=4400$, since its trusted prefix is typically longer.

\subsection{Baselines}\label{appendix_baseline}

\begin{itemize}[leftmargin=*,parsep=4pt,topsep=4pt]
\item \textbf{OPD}~\citep{lu2025onpolicy} performs on-policy distillation by sampling trajectories from the current student policy and using a frozen teacher to provide token-level supervision on the student-visited states. We implement OPD following Eq.~\eqref{eq:opd} without adaptive truncation, using the response-length limits specified in Table~\ref{tab:pairs}.
\item \textbf{TOPD}~\citep{zhang2026full} accelerates OPD by truncating student rollouts to a fixed horizon, thereby reducing the number of generated and supervised tokens. We use a fixed rollout limit of $4096$ tokens and keep its evaluation budget identical to that of OPD.
\item \textbf{Prune-OPD}~\citep{yang2026prune} dynamically adjusts a batch-shared rollout horizon based on feedback from previous training steps, reducing unnecessary computation while allowing the horizon to evolve during training. We use threshold $0.7$, $w_{\mathrm{drop}}=0.01$, $w_{\mathrm{base}}=0.5$, initial and minimum rollout lengths of $1024$ tokens, and a hit ratio of $0.1$.
\end{itemize}

\subsection{Protocol and Hyperparameters}\label{appendix_protocol}

\input{tables/parameter}

\textbf{Training and evaluation.} Table~\ref{tab:Hyperparameters} lists the optimization, controller and decoding settings. Each training batch contains $64$ prompts with $4$ responses each, giving $256$ trajectories. Accuracy is the mean success rate over sampled responses, macro-averaged across benchmarks.

\textbf{Hardware and timing.} Each run uses one node with $8$ NVIDIA H20 GPUs (approximately $96$\,GB each), two AMD EPYC 9K84 96-core CPUs and approximately $2.2$\,TiB of host memory. Reported training time sums training-step durations and excludes evaluation.

%% file: tables/dataset.tex
\begin{table}[htbp!]
\centering
\small
\setlength{\tabcolsep}{4pt}
\renewcommand{\arraystretch}{1.15}
\caption{The training corpus and the five in-domain benchmarks are competition mathematics; the four out-of-domain benchmarks are not, and GPQA is not mathematics at all.}
\label{tab.datasets}
\resizebox{1.0\linewidth}{!}{
\begin{tabular}{llcrc}
\toprule
\textbf{Dataset} & \textbf{Domain} & \textbf{Split used} & \textbf{\#Problems} & \textbf{Evaluation budget} \\
\midrule
\rowcolor{gray!10} \multicolumn{5}{l}{\textit{Training corpus}} \\
DAPO-Math-17K~\citep{yu2025dapo} & Competition mathematics & train (all) & 17{,}917 & --- \\
\midrule
\rowcolor{gray!10} \multicolumn{5}{l}{\textit{In-domain evaluation}} \\
AMC23 & Competition mathematics & test & 83 & $31744$ tok., mean@16 \\
AIME24 & Competition mathematics & test & 30 & $31744$ tok., mean@16 \\
AIME25 & Competition mathematics & test & 30 & $31744$ tok., mean@16 \\
HMMT24 & Competition mathematics & test & 30 & $31744$ tok., mean@16 \\
HMMT25 & Competition mathematics & test & 30 & $31744$ tok., mean@16 \\
\cmidrule(lr){1-5}
\textit{total} & & & \textit{203} & \\
\midrule
\rowcolor{gray!10} \multicolumn{5}{l}{\textit{Out-of-domain evaluation}} \\
MATH-500 & Mathematics, mixed level & test & 250 & $31744$ tok., mean@4 \\
Minerva & Undergraduate STEM & test & 250 & $31744$ tok., mean@4 \\
Olympiad-Bench & Olympiad mathematics & test & 250 & $31744$ tok., mean@4 \\
GPQA & Graduate-level science & test & 198 & $31744$ tok., mean@4 \\
\cmidrule(lr){1-5}
\textit{total} & & & \textit{948} & \\
\bottomrule
\end{tabular}
}
\end{table}

%% file: tables/pairs.tex
\begin{table}[htbp!]
\centering
\small
\setlength{\tabcolsep}{4pt}
\renewcommand{\arraystretch}{1.5}
\caption{The four teacher--student pairs and the controller each one runs.}
\label{tab:pairs}
\resizebox{1.0\linewidth}{!}{
\begin{tabular}{llcrrrrr}
\toprule
\textbf{Student} & \textbf{Teacher} & \textbf{Relationship} & \textbf{Max resp.} & $k_0$ & $W$ & $k_{\min}$ & $L_{\max}$ \\
\midrule
R1-Distill-Qwen-1.5B & JustRL-DeepSeek-1.5B & Same-size RL-improved copy of the student & 12288 & 1024 & 1024 & 512 & 4400 \\
R1-Distill-Qwen-1.5B & R1-Distill-Qwen-7B & Larger same-family reasoning model & 12288 & 128 & 128 & 16 & 12288 \\
Qwen3-1.7B-Base & Qwen3-4B (Non-thinking) & Stronger same-family teacher, base student & 8192 & 128 & 128 & 16 & 8192 \\
Qwen3-4B-Base & Qwen3-4B (Non-thinking) & Same-size post-trained teacher, base student & 12288 & 128 & 128 & 16 & 12288 \\
\bottomrule
\end{tabular}
}
\end{table}

%% file: tables/parameter.tex
\begin{table}[htbp!]
\centering
\small
\setlength{\tabcolsep}{4pt}
\renewcommand{\arraystretch}{0.9}
\caption{Parameter setting of our method \method.}
\vspace{-2mm}
\label{tab:Hyperparameters}
\resizebox{1.0\linewidth}{!}{
\begin{tabular}{lc}
\toprule
\textbf{Parameter} & \textbf{Value}\\
\midrule
\rowcolor{gray!10} \multicolumn{2}{l}{\textit{Optimisation}} \\
Train batch size (prompts per step) & $64$ \\
PPO mini-batch size & $64$ (one update per step) \\
Samples per prompt $n$ & $4$ \\
Learning rate & $1\times10^{-6}$, constant  \\
Micro-batch per GPU & $1$, dynamic batching enabled  \\
Total steps / epochs & $203$ / $1$  \\
Max prompt length & $1024$ tokens  \\
Max response length (rollout) & $8192$ on G3, $12288$ on G1, G2, G4  \\
Parameter dtype & fp32  \\
Loss aggregation & token-mean  \\
KL penalty & disabled \\
Prompt shuffling & enabled \\
\midrule
\rowcolor{gray!10} \multicolumn{2}{l}{\textit{OPD objective}, Eq.~\eqref{eq:opd}} \\
Advantage estimator & \texttt{token\_reward\_direct} \\
Candidate set size $k$ & $16$ \\
Candidate set source & student top-$k$ (\texttt{only\_stu}) \\
Reward weight mode & \texttt{student\_p} \\
Student temperature & $1.0$ \\
Teacher temperature & $1.0$ \\
Repetition penalty & $1.0$ \\
Format reward & disabled \\
\midrule
\rowcolor{gray!10} \multicolumn{2}{l}{\textit{ICT controller}, Eqs.~\eqref{eq:phi}--\eqref{eq:usable} and~\eqref{eq:step}} \\
Drift threshold $\eta$ & $0.70$ \\
Allowance $B$ & $200$ \\
Initial checkpoint $k_0$ & $128$; $1024$ on G1 \\
Trailing window $W$ & $128$; $1024$ on G1 \\
Minimum advance $k_{\min}$ & $16$; $512$ on G1 \\
Step discount $\rho$ & $1.00$ on all four pairs \\
Context limit $L_{\max}$ & max response length; $4400$ on G1 \\
Verification micro-batch & $4096$ tokens \\
Sampler resident across rounds & yes, at \texttt{gpu\_memory\_utilization}$=0.6$ \\
\midrule
\rowcolor{gray!10} \multicolumn{2}{l}{\textit{Evaluation}} \\
Decoding budget & $31744$ tokens \\
Samples per prompt & $16$ in-domain, $4$ out-of-domain \\
Sampling & \texttt{do\_sample}, $T=1.0$, top-$p=0.95$ \\
Cadence & every $20$ steps; once at step $203$ for ablations \\
Reported metric & mean@$n$ accuracy, macro-averaged over benchmarks \\
\bottomrule
\end{tabular}
}
\end{table}

%% file: appendix/4_moreresult.tex
\section{Additional Results}\label{appendix_moreresult}

This section provides the numerical results and trajectory statistics supporting the efficiency, mechanism, transferability and robustness studies in Section~\ref{sec:experiments}.

\subsection{Training-Time Breakdown}\label{appendix_efficiency}

\input{tables/efficiency}

Table~\ref{tab:efficiency} gives the values underlying Figure~\ref{fig.steptime}. Shorter rollouts reduce generation and policy-update time on all four pairs. Response-token reductions exceed step-time reductions because prompt processing, repeated verification and other per-step costs remain. The step-time ratios range from $2.20\times$ to $7.52\times$, compared with response-token reductions of $2.3\times$ to $21.9\times$.

\subsection{Trajectory Lengths and Stopping Statistics}\label{appendix_mechanism}

\input{tables/mechanism}

Table~\ref{tab:mechanism} summarizes retained lengths, stopping reasons and verification costs over $203$ training steps. The mean retained length ranges from $304$ to $2708$ tokens across pairs, with substantial variation within each batch. Figure~\ref{fig.boundary} shows how these lengths evolve during training.

The allowance determines most stopping decisions on G2--G4. On G1, an average batch stops $130$ trajectories at the allowance, $96$ at the length limit and $30$ at EOS. Across all four pairs, verification processes $1.54\times$--$2.41\times$ as many positions as are retained, and decoding overshoot is $4.5\%$--$10.9\%$ of the retained length.

\begin{figure}[htbp!]
    \centering
    \includegraphics[width=1.0\textwidth]{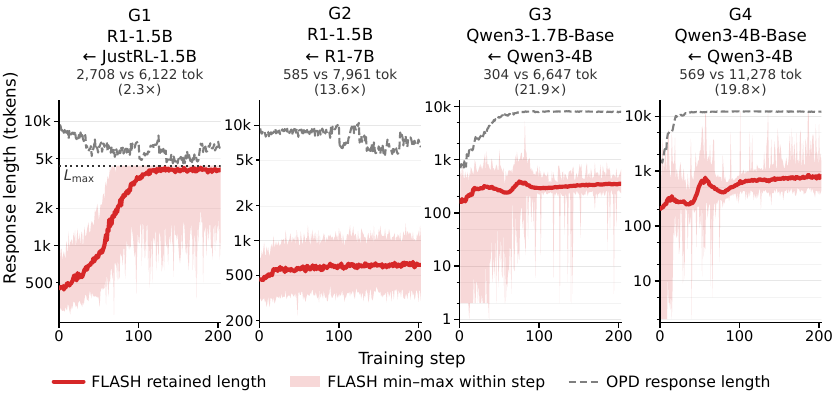}
    \caption{Retained lengths over training.}
    \label{fig.boundary}
\end{figure}

\subsection{Out-of-Domain Transfer}\label{appendix_ood}

\input{tables/ood}

Table~\ref{tab:ood} gives the results for Figure~\ref{fig.transfer}. Macro differences on G1--G3 are within $1.4$ points. G4 improves on all four benchmarks, with gains of $5.70$, $5.30$, $4.10$ and $8.25$ points on MATH-500, Minerva, Olympiad-Bench and GPQA, respectively. G1 has lower scores on Olympiad-Bench and GPQA, by $2.90$ and $2.54$ points.

\subsection{Cross-Generation Transfer}\label{appendix_crossfamily}

\input{tables/crossfamily}

Table~\ref{tab:crossfamily} report cross-generation results. On R1-1.5B$\leftarrow$Qwen3-4B, we reaches $28.45$ average accuracy compared with $28.34$ for OPD, while reducing training time from $12.02$ to $1.38$ hours. 

%% file: tables/efficiency.tex
\begin{table*}[htbp!]
    \centering
    \small
    \renewcommand\arraystretch{1.1}
    \setlength{\tabcolsep}{8pt}
    \caption{Where the speedup comes from: mean wall-clock seconds per training step.}
    \label{tab:efficiency}
    \resizebox{1.0\linewidth}{!}{
    \begin{tabular}{lcccccccc}
       \toprule
       \multirow{2}{*}{\textbf{Method}} & \multicolumn{5}{c|}{\textbf{Step-time attribution (s/step)}} & \multicolumn{3}{c}{\textbf{Tokens vs.\ time}} \\
       \cmidrule(lr){2-6} \cmidrule(lr){7-9}
       & Rollout & Teacher & Log-prob & Update & \textbf{Step} & Length & Token $\downarrow$ & Step $\downarrow$ \\
       \midrule
       \rowcolor{gray!10} \multicolumn{9}{c}{\textit{R1-1.5B $\leftarrow$ JustRL-1.5B}} \\
       OPD & 70.3 & 11.2 & 10.5 & 31.5 & 132.3 & 6122 & -- & -- \\
       \rowcolor{cyan!8} {\method} & 26.2 & 14.2 & 0.0 & 14.7 & {60.0} & 2708 & 2.3$\times$ & {2.20$\times$} \\
       \midrule
       \rowcolor{gray!10} \multicolumn{9}{c}{\textit{R1-1.5B $\leftarrow$ R1-7B}} \\
       OPD & 88.8 & 40.4 & 13.3 & 42.4 & 194.3 & 7961 & -- & -- \\
       \rowcolor{cyan!8} {\method} & 14.4 & 12.2 & 0.0 & 4.3 & {36.5} & 585 & 13.6$\times$ & {5.32$\times$} \\
       \midrule
       \rowcolor{gray!10} \multicolumn{9}{c}{\textit{Qwen3-1.7B-Base $\leftarrow$ Qwen3-4B}} \\
       OPD & 109.7 & 25.8 & 12.6 & 41.1 & 196.1 & 6647 & -- & -- \\
       \rowcolor{cyan!8} {\method} & 11.1 & 10.1 & 0.0 & 2.7 & {27.5} & 304 & 21.9$\times$ & {7.13$\times$} \\
       \midrule
       \rowcolor{gray!10} \multicolumn{9}{c}{\textit{Qwen3-4B-Base $\leftarrow$ Qwen3-4B}} \\
       OPD & 306.8 & 47.8 & 45.6 & 175.7 & 586.9 & 11278 & -- & -- \\
       \rowcolor{cyan!8} {\method} & 44.2 & 19.7 & 0.0 & 8.5 & {78.1} & 569 & 19.8$\times$ & {7.52$\times$} \\
       \bottomrule
    \end{tabular}
    }
\end{table*}

%% file: tables/mechanism.tex
\begin{table*}[htbp!]
    \centering
    \small
    \renewcommand\arraystretch{1.3}
    \setlength{\tabcolsep}{4pt}
    \caption{Mechanism counters, averaged over all training steps.}
    \label{tab:mechanism}
    \resizebox{1.0\linewidth}{!}{
    \begin{tabular}{lccccccccc}
       \toprule
       \multirow{2}{*}{\textbf{Pair}} & \multicolumn{3}{c}{\textbf{Retained length $\hat{E}$}} & \multicolumn{3}{c}{\textbf{Stopping rule (of 256)}} & \multicolumn{3}{c}{\textbf{Verification cost}} \\
       \cmidrule(lr){2-4} \cmidrule(lr){5-7} \cmidrule(lr){8-10}
       & min & mean & max & allowance & $L_{\max}$ & EOS & rounds & amplif. & overshoot \\
       \midrule
       R1-1.5B $\leftarrow$ JustRL-1.5B & 830 & 2708 & 3509 & 130 & 96 & 30 & 2.64 & 1.54$\times$ & 231 (8.5\%) \\
       R1-1.5B $\leftarrow$ R1-7B & 335 & 585 & 1052 & 256 & 0 & 0 & 6.23 & 2.41$\times$ & 26 (4.5\%) \\
       Qwen3-1.7B-Base $\leftarrow$ Qwen3-4B & 126 & 304 & 672 & 243 & 0 & 13 & 4.98 & 2.23$\times$ & 14 (4.5\%) \\
       Qwen3-4B-Base $\leftarrow$ Qwen3-4B & 293 & 569 & 2117 & 250 & 0 & 6 & 6.23 & 2.19$\times$ & 62 (10.9\%) \\
       \bottomrule
    \end{tabular}
    }
    \vspace{-6mm}
\end{table*}

%% file: tables/ood.tex

\newcommand{\hlcell}[1]{\cellcolor{cyan!8}#1}

\begin{table}[htbp!]
    \centering
    \small
    \renewcommand{\arraystretch}{1.1}
    \setlength{\tabcolsep}{4pt}
    \caption{Out-of-domain transfer.}
    \label{tab:ood}

    \resizebox{1.0\linewidth}{!}{
    \begin{tabular}{llcccccc}
        \toprule
        \textbf{Pair}
        & \textbf{Method}
        & \textbf{MATH-500}
        & \textbf{Minerva}
        & \textbf{Olympiad}
        & \textbf{GPQA}
        & \textbf{Avg.}
        & \textbf{$\Delta$} \\
        \midrule

        \multirow{2}{*}{R1-1.5B $\leftarrow$ JustRL-1.5B}
        & OPD
        & 90.40
        & 38.40
        & 61.60
        & 36.93
        & 56.83
        & \\

        & \hlcell{\method}
        & \hlcell{90.60}
        & \hlcell{38.10}
        & \hlcell{58.70}
        & \hlcell{34.39}
        & \hlcell{55.45}
        & \hlcell{tie} \\
        \midrule

        \multirow{2}{*}{R1-1.5B $\leftarrow$ R1-7B}
        & OPD
        & 85.20
        & 34.50
        & 51.40
        & 11.80
        & 45.73
        & \\

        & \hlcell{\method}
        & \hlcell{85.20}
        & \hlcell{34.30}
        & \hlcell{50.00}
        & \hlcell{15.23}
        & \hlcell{46.18}
        & \hlcell{tie} \\
        \midrule

        \multirow{2}{*}{Qwen3-1.7B-Base $\leftarrow$ Qwen3-4B}
        & OPD
        & 57.60
        & 22.80
        & 26.80
        & 15.36
        & 30.64
        & \\

        & \hlcell{\method}
        & \hlcell{59.80}
        & \hlcell{23.30}
        & \hlcell{29.30}
        & \hlcell{13.96}
        & \hlcell{31.59}
        & \hlcell{tie} \\
        \midrule

        \multirow{2}{*}{Qwen3-4B-Base $\leftarrow$ Qwen3-4B}
        & OPD
        & 73.50
        & 32.90
        & 43.60
        & 16.24
        & 41.56
        & \\

        & \hlcell{\method}
        & \hlcell{79.20}
        & \hlcell{38.20}
        & \hlcell{47.70}
        & \hlcell{24.49}
        & \hlcell{47.40}
        & \hlcell{$+5.84$} \\
        \bottomrule
    \end{tabular}
    }
\end{table}

%% file: tables/crossfamily.tex


\begin{table}[htbp!]
\centering
\small
\setlength{\tabcolsep}{5pt}
\renewcommand{\arraystretch}{1.2}
\caption{Cross-generation transfer.}
\label{tab:crossfamily}

\resizebox{1.0\linewidth}{!}{
\begin{tabular}{llrrrrrcrr}
\toprule
\textbf{Pair}
& \textbf{Method}
& AMC23
& AIME24
& AIME25
& HMMT24
& HMMT25
& \textbf{Avg.}
& \textbf{Time}
& \textbf{Speed Up} \\
\midrule

\multirow{3}{*}{R1-1.5B $\leftarrow$ Qwen3-4B}
& Student
& 63.35
& 26.80
& 20.00
& 12.40
& 12.30
& 26.97
& --
& -- \\

& OPD
& 60.02
& 31.25
& 22.92
& 12.92
& 14.58
& 28.34
& 12.02h
& 1.00$\times$ \\

& \hlcell{\method{}}
& \hlcell{64.76}
& \hlcell{27.92}
& \hlcell{22.08}
& \hlcell{13.96}
& \hlcell{13.54}
& \hlcell{28.45}
& \hlcell{1.38h}
& \hlcell{8.71$\times$} \\

\bottomrule
\end{tabular}
}
\end{table}

%% file: appendix/5_morediscussion.tex
\section{Limitations}\label{more_discuss}\label{sec:disc-limitations}

While our results demonstrate that \method{} can substantially accelerate on-policy distillation while preserving competitive performance, several limitations warrant further investigation.

\ding{182} \textbf{Compatibility criterion.}
Our current method uses top-$k$ overlap as a practical proxy for measuring local supervision compatibility between the student and teacher. Accordingly, the stopping guarantee characterizes the boundary induced by this criterion rather than directly guaranteeing downstream task accuracy. We adopt this criterion because it can be evaluated online at intermediate prefixes without requiring completed responses, making it particularly suitable for adaptive rollout control. Nevertheless, supervision compatibility may admit richer definitions beyond top-$k$ agreement. Exploring alternative signals that better capture when teacher supervision remains beneficial, while retaining efficient online evaluation, is an important direction for future work.
\ding{184} \textbf{Beyond token-level supervision.}
The current formulation is designed for token-level OPD, where supervision compatibility can be evaluated before a trajectory is completed. Extending the same adaptive stopping principle to sequence-level supervision is less straightforward, since outcome rewards or sequence-level judgments generally depend on the completed response. Consequently, directly applying our stopping rule could remove information required to evaluate the final outcome. Developing analogous notions of intermediate reliability for sequence-level objectives may therefore enable adaptive computation beyond OPD and represents an interesting direction.

%% file: ref.bib
@String(JMLR = {JMLR})

@article{li2026rethinking,
  title={Rethinking on-policy distillation of large language models: Phenomenology, mechanism, and recipe},
  author={Li, Yaxuan and Zuo, Yuxin and He, Bingxiang and Zhang, Jinqian and Xiao, Chaojun and Qian, Cheng and Yu, Tianyu and Gao, Huan-ang and Yang, Wenkai and Liu, Zhiyuan and others},
  journal={arXiv preprint arXiv:2604.13016},
  year={2026}
}

@article{fu2026revisiting,
  title={Revisiting on-policy distillation: Empirical failure modes and simple fixes},
  author={Fu, Yuqian and Huang, Haohuan and Jiang, Kaiwen and Liu, Jiacai and Jiang, Zhuo and Zhu, Yuanheng and Zhao, Dongbin},
  journal={arXiv preprint arXiv:2603.25562},
  year={2026}
}

@inproceedings{gu2024minillm,
  title={Minillm: Knowledge distillation of large language models},
  author={Gu, Yuxian and Dong, Li and Wei, Furu and Huang, Minlie},
  booktitle={International Conference on Learning Representations},
  volume={2024},
  pages={32694--32717},
  year={2024}
}

@article{hinton2015distilling,
  title={Distilling the knowledge in a neural network},
  author={Hinton, Geoffrey and Vinyals, Oriol and Dean, Jeff},
  journal={arXiv preprint arXiv:1503.02531},
  year={2015}
}

@article{hubotter2026reinforcement,
  title={Reinforcement learning via self-distillation},
  author={H{\"u}botter, Jonas and L{\"u}beck, Frederike and Behric, Lejs and Baumann, Anton and Bagatella, Marco and Marta, Daniel and Hakimi, Ido and Shenfeld, Idan and Buening, Thomas Kleine and Guestrin, Carlos and others},
  journal={arXiv preprint arXiv:2601.20802},
  year={2026}
}

@inproceedings{jang2026stable,
  title={Stable on-policy distillation through adaptive target reformulation},
  author={Jang, Ijun and Yeom, Jewon and Yeo, Juan and Lim, Hyunggyu and Kim, Taesup},
  booktitle={Findings of the Association for Computational Linguistics: ACL 2026},
  pages={42217--42227},
  year={2026}
}

@inproceedings{jiao2020tinybert,
  title={TinyBERT: Distilling BERT for Natural Language Understanding},
  author={Jiao, Xiaoqi and Yin, Yichun and Shang, Lifeng and Jiang, Xin and Chen, Xiao and Li, Linlin and Wang, Fang and Liu, Qun},
  booktitle={Findings of the Association for Computational Linguistics: EMNLP 2020},
  pages={4163--4174},
  year={2020}
}

@misc{jin2026entropy,
  title={Entropy-Aware On-Policy Distillation of Language Models},
  author={Jin, Woogyeol and Min, Taywon and Yang, Yongjin and Kadhe, Swanand Ravindra and Zhou, Yi and Wei, Dennis and Baracaldo, Nathalie and Lee, Kimin},
  year={2026},
  eprint={2603.07079},
  archivePrefix={arXiv}
}

@article{kim2026does,
  title={Why Does Self-Distillation (Sometimes) Degrade the Reasoning Capability of LLMs?},
  author={Kim, Jeonghye and Luo, Xufang and Kim, Minbeom and Lee, Sangmook and Kim, Dohyung and Jeon, Jiwon and Li, Dongsheng and Yang, Yuqing},
  journal={arXiv preprint arXiv:2603.24472},
  year={2026}
}

@inproceedings{kim2016sequence,
  title={Sequence-level knowledge distillation},
  author={Kim, Yoon and Rush, Alexander M},
  booktitle={Proceedings of the 2016 conference on empirical methods in natural language processing},
  pages={1317--1327},
  year={2016}
}

@article{ko2026scaling,
  title={Scaling reasoning efficiently via relaxed on-policy distillation},
  author={Ko, Jongwoo and Abdali, Sara and Kim, Young Jin and Chen, Tianyi and Cameron, Pashmina},
  journal={arXiv preprint arXiv:2603.11137},
  year={2026}
}

@inproceedings{li2025small,
  title={Small models struggle to learn from strong reasoners},
  author={Li, Yuetai and Yue, Xiang and Xu, Zhangchen and Jiang, Fengqing and Niu, Luyao and Lin, Bill Yuchen and Ramasubramanian, Bhaskar and Poovendran, Radha},
  booktitle={Findings of the Association for Computational Linguistics: ACL 2025},
  pages={25366--25394},
  year={2025}
}

@article{sang2026crispcompressedreasoningiterative,
  title={Crisp: Compressed reasoning via iterative self-policy distillation},
  author={Sang, Hejian and Xu, Yuanda and Zhou, Zhengze and He, Ran and Wang, Zhipeng and Sun, Jiachen},
  journal={arXiv preprint arXiv:2603.05433},
  year={2026}
}

@article{sanh2019distilbert,
  title={DistilBERT, a distilled version of BERT: smaller, faster, cheaper and lighter},
  author={Sanh, Victor and Debut, Lysandre and Chaumond, Julien and Wolf, Thomas},
  journal={arXiv preprint arXiv:1910.01108},
  year={2019}
}

@article{wang2020minilm,
  title={Minilm: Deep self-attention distillation for task-agnostic compression of pre-trained transformers},
  author={Wang, Wenhui and Wei, Furu and Dong, Li and Bao, Hangbo and Yang, Nan and Zhou, Ming},
  journal={Advances in neural information processing systems},
  volume={33},
  pages={5776--5788},
  year={2020}
}

@misc{yang2025qwen3,
  title={Qwen3 Technical Report},
  author={Yang, An and Li, Anfeng and Yang, Baosong and Zhang, Beichen and Hui, Binyuan and Zheng, Bo and Yu, Bowen and Gao, Chang and Huang, Chengen and Lv, Chenxu and others},
  year={2025},
  eprint={2505.09388},
  archivePrefix={arXiv}
}

@article{yang2026learning,
  title={Learning beyond teacher: Generalized on-policy distillation with reward extrapolation},
  author={Yang, Wenkai and Liu, Weijie and Xie, Ruobing and Yang, Kai and Yang, Saiyong and Lin, Yankai},
  journal={arXiv preprint arXiv:2602.12125},
  year={2026}
}

@article{ye2026online,
  title={Online experiential learning for language models},
  author={Ye, Tianzhu and Dong, Li and Dong, Qingxiu and Wu, Xun and Huang, Shaohan and Wei, Furu},
  journal={arXiv preprint arXiv:2603.16856},
  year={2026}
}

@article{yu2025dapo,
  title={Dapo: An open-source llm reinforcement learning system at scale},
  author={Yu, Qiying and Zhang, Zheng and Zhu, Ruofei and Yuan, Yufeng and Zuo, Xiaochen and Yue, Yu and Dai, Weinan and Fan, Tiantian and Liu, Gaohong and Liu, Lingjun and others},
  journal={Advances in Neural Information Processing Systems},
  volume={38},
  pages={113222--113244},
  year={2026}
}

@article{accordion,
  title={Accordion-thinking: Self-regulated step summaries for efficient and readable llm reasoning},
  author={Yang, Zhicheng and Guo, Zhijiang and Huang, Yinya and Wang, Yongxin and Shi, Wenlei and Wang, Yiwei and Liang, Xiaodan and Tang, Jing},
  journal={arXiv preprint arXiv:2602.03249},
  year={2026}
}

@article{zhang2026full,
  title={Are Full Rollouts Necessary for On-Policy Distillation?},
  author={Zhang, Yaocheng and Chai, Jiajun and Fu, Yuqian and Tu, Songjun and Wang, Xiaohan and Lin, Wei and Yin, Guojun and Zhang, Qichao and Zhu, Yuanheng and Zhao, Dongbin},
  journal={arXiv preprint arXiv:2605.31490},
  year={2026}
}

@inproceedings{ross2011reduction,
  title={A reduction of imitation learning and structured prediction to no-regret online learning},
  author={Ross, St{\'e}phane and Gordon, Geoffrey and Bagnell, Drew},
  booktitle={Proceedings of the fourteenth international conference on artificial intelligence and statistics},
  pages={627--635},
  year={2011},
  organization={JMLR Workshop and Conference Proceedings}
}

@inproceedings{agarwal2024policy,
  title={On-policy distillation of language models: Learning from self-generated mistakes},
  author={Agarwal, Rishabh and Vieillard, Nino and Zhou, Yongchao and Stanczyk, Piotr and Ramos Garea, Sabela and Geist, Matthieu and Bachem, Olivier},
  booktitle={International Conference on Learning Representations},
  volume={2024},
  pages={21246--21263},
  year={2024}
}

@misc{lu2025onpolicy,
  title  = {On-Policy Distillation},
  author = {Lu, Kevin and {Thinking Machines Lab}},
  year   = {2025},
  note   = {Thinking Machines Lab: Connectionism},
  doi    = {10.64434/tml.20251026},
  url    = {https://thinkingmachines.ai/blog/on-policy-distillation}
}

@article{zeng2026glm5,
  title={Glm-5: from vibe coding to agentic engineering},
  author={Zeng, Aohan and Lv, Xin and Hou, Zhenyu and Du, Zhengxiao and Zheng, Qinkai and Chen, Bin and Yin, Da and Ge, Chendi and Huang, Chenghua and Xie, Chengxing and others},
  journal={arXiv preprint arXiv:2602.15763},
  year={2026}
}

@article{shenfeld2026selfdistillation,
  title   = {Self-Distillation Enables Continual Learning},
  author  = {Shenfeld, Idan and Damani, Mehul and H{\"u}botter, Jonas and Agrawal, Pulkit},
  journal = {arXiv preprint arXiv:2601.19897},
  year    = {2026}
}

@article{zhao2026selfdistilled,
  title={Self-distilled reasoner: On-policy self-distillation for large language models},
  author={Zhao, Siyan and Xie, Zhihui and Liu, Mengchen and Huang, Jing and Pang, Guan and Chen, Feiyu and Grover, Aditya},
  journal={arXiv preprint arXiv:2601.18734},
  year={2026}
}

@article{wu2026lightningopd,
  title={Lightning opd: Efficient post-training for large reasoning models with offline on-policy distillation},
  author={Wu, Yecheng and Han, Song and Cai, Hai},
  journal={arXiv preprint arXiv:2604.13010},
  year={2026}
}

@article{cai2026learning,
  title={Learning to foresee: Unveiling the unlocking efficiency of on-policy distillation},
  author={Cai, Yuchen and Cao, Ding and Lin, Liang and Luo, Chunxi and Xu, Xin and Yang, Kai and Liu, Weijie and Yang, Saiyong and Zhao, Tianxiang and Sun, Guangzhong and others},
  journal={arXiv preprint arXiv:2605.11739},
  year={2026}
}

@inproceedings{zhang2026fast,
  title={Fast and effective on-policy distillation from reasoning prefixes},
  author={Zhang, Dongxu and Yang, Zhichao and Janghorbani, Sepehr and Han, Jun and Ressler II, Andrew and Qian, Qian and Lyng, Gregory D and Batra, Sanjit Singh and Tillman, Robert E},
  booktitle={Findings of the Association for Computational Linguistics: ACL 2026},
  pages={25553--25569},
  year={2026}
}

@article{yang2026prune,
  title   = {Prune-OPD: Efficient and Reliable On-Policy Distillation for Long-Horizon Reasoning},
  author  = {Yang, Zhicheng and Guo, Zhijiang and Song, Yifan and Xu, Minrui and Wang, Yongxin and Wang, Yiwei and Liang, Xiaodan and Tang, Jing},
  journal = {arXiv preprint arXiv:2605.07804},
  year    = {2026}
}

@article{liang2026adwin,
  title={ADWIN: Adaptive Windows for Horizon-Aware On-Policy Distillation},
  author={Liang, Kun and Tang, Chenming and Bai, Clive and Liu, Weijie and Yang, Saiyong and Wu, Yunfang},
  journal={arXiv preprint arXiv:2605.28396},
  year={2026}
}

@article{zhao2026prefix,
  title={Prefix-Guided On-Policy Distillation: Mining Golden Trajectories from Rollouts},
  author={Zhao, Qingfei and Song, Huan and Tian, Shuyu and Shao, Jiawei and Li, Xuelong},
  journal={arXiv preprint arXiv:2606.21994},
  year={2026}
}

@article{xu2026deepseek,
  title={DeepSeek-V4. 1-Flash: Pushing the Limits of KV Cache Compression},
  author={Xu, Anyi and Li, B and Lin, Bangcai and Xue, Bing and Xian, BingCheng and Xu, Bingzheng and Wu, Bochao and Zhang, Bowei and Deng, Boyi and Yu, CC and others},
  journal={arXiv preprint arXiv:2609.19969},
  year={2026}
}

@article{wang2026demystifying,
  title={Demystifying On-Policy Distillation: Roles, Pathologies, and Regulations},
  author={Wang, Rui and Wang, Hongru and Chen, Yi and Xue, Boyang and Fang, Tianqing and Yu, Wenhao and Wong, Kam-Fai},
  journal={arXiv preprint arXiv:2607.13399},
  year={2026}
}
